\documentclass{article}

\usepackage{microtype}
\usepackage{graphicx}
\usepackage{booktabs, subcaption} % for professional tables
\usepackage{natbib,setspace}
\usepackage{soul}
\usepackage{multicol, multirow}
\usepackage{array}
\usepackage{subcaption}

\usepackage{hyperref}

\usepackage[accepted]{icml2024}

\usepackage{amsmath}
\usepackage{amssymb}
\usepackage{mathtools}
\usepackage{amsthm}
\usepackage{enumitem}

\usepackage[capitalize,noabbrev]{cleveref}

\theoremstyle{plain}
\newtheorem{theorem}{Theorem}[section]

\newtheorem{lemma}[theorem]{Lemma}

\theoremstyle{definition}

\theoremstyle{remark}

\providecommand{\jw}{{}}
\providecommand{\csiszar}{Csisz\`{a}r}

\usepackage[textsize=tiny]{todonotes}

\icmltitlerunning{Scalable Wasserstein Gradient Flow for Generative Modeling through Unbalanced Optimal Transport}

\begin{document}

\twocolumn[
\icmltitle{Scalable Wasserstein Gradient Flow for Generative Modeling through Unbalanced Optimal Transport}



\icmlsetsymbol{equal}{*}

\begin{icmlauthorlist}
\icmlauthor{Jaemoo Choi}{equal,snu}
\icmlauthor{Jaewoong Choi}{equal,kias}
\icmlauthor{Myungjoo Kang}{snu}
\end{icmlauthorlist}

\icmlaffiliation{snu}{Seoul National University}
\icmlaffiliation{kias}{Korea Institute for Advanced Study}

\icmlcorrespondingauthor{Myungjoo Kang}{mkang@snu.ac.kr}

\icmlkeywords{Machine Learning, ICML}

\vskip 0.3in
]



\printAffiliationsAndNotice{\icmlEqualContribution} % otherwise use the standard text.

\begin{abstract}
Wasserstein Gradient Flow (WGF) describes the gradient dynamics of probability density within the Wasserstein space. WGF provides a promising approach for conducting optimization over the probability distributions. Numerically approximating the continuous WGF requires the time discretization method. The most well-known method for this is the JKO scheme. In this regard, previous WGF models employ the JKO scheme and parametrize transport map for each JKO step. However, this approach results in quadratic training complexity $O(K^2)$ with the number of JKO step $K$. This severely limits the scalability of WGF models. In this paper, we introduce a scalable WGF-based generative model, called Semi-dual JKO (S-JKO). Our model is based on the semi-dual form of the JKO step, derived from the equivalence between the JKO step and the Unbalanced Optimal Transport. Our approach reduces the training complexity to  $O(K)$. We demonstrate that our model significantly outperforms existing WGF-based generative models, achieving FID scores of 2.62 on CIFAR-10 and 5.46 on CelebA-HQ-256, which are comparable to state-of-the-art image generative models.
\end{abstract}

\section{Introduction}
% Generative models are a class of Deep Learning models that learn the underlying distribution of training data. These models fall into two categories based on their approach: (1) models that directly approximate the target distribution (Explicit Density Models) and (2) models that approximate sampling from the target distribution (Implicit Density Models). Explicit density models include Energy-based models, Diffusion models, Variational Autoencoders, and Flow models. On the other hand, Implicit density models include Generative Adversarial Networks and Optimal Transport Maps. 
Generative models are a class of Deep Learning models that learn the underlying distribution of training data. There are diverse approaches for generative modeling, such as Energy-based models \cite{ebm, implicitebm}, Diffusion models \cite{ddpm, scoresde}, Variational Autoencoders \cite{kingma2013auto}, Flow models \citep{dinh2016density, kingma2018glow}, Generative Adversarial Networks \cite{gan}, Optimal Transport Maps \cite{otm, uotm}, and Wasserstein Gradient Flows. Recently, generative models achieved impressive progress, demonstrating the ability to produce high-quality samples on high-resolution image datasets.
Despite these advancements, Wasserstein Gradient Flow models still face challenges in scalability to high-dimensional image datasets.

Wasserstein Gradient Flow (WGF) investigates the minimizing dynamics of probability density following the steepest descent direction of a given functional.  
WGF plays an important role across various areas involving optimization over probability densities, e.g.  Optimal Transport (OT) \cite{Fillippo, carlier2017convergence}, Physics \cite{carrillo2022primal,adams2011large}, Machine learning \cite{lin2021wasserstein, jkoex2}, and Sampling \cite{bernton2018langevin,frogner1806approximate,svgd,svgdchi,kale,particle}. The Jordan-Kinderlehrer-Otto (JKO) scheme is a prominent method for numerically approximating WGF \cite{jko}. The JKO scheme corresponds to the time discretization of WGF. The previous works utilized the JKO scheme and conducted optimization for every transport map at each JKO step \cite{jkoex2, wgfkorotin, jkoex1, population, vwgf}. However, this approach incurs quadratic training complexity $O(K^2)$ with the number of JKO step K. This quadratic complexity arises from the necessity to simulate the entire trajectory of the JKO scheme. This complexity significantly limited the scalability of WGF models through the prolonged training time and the limited model size for parametrization.

To overcome these challenges, we suggest a new generative algorithm by utilizing the semi-dual form of the JKO step. We refer to our model as the Semi-dual JKO (\textbf{\textit{S-JKO}}). Our model consists of two components. First, we introduce the semi-dual form of the JKO step from the equivalence between the JKO step and the Unbalanced Optimal Transport problem \cite{uot1, uot2} (Sec \ref{sec:relationship}). Second, we introduce the reparametrization trick to handle the complexity challenges of existing JKO models (Sec \ref{sec:algorithm}). Our model reduces the training complexity from quadratic to linear $O(K)$. Our model achieves significantly improved scalability compared to existing WGF-based models. Specifically, our S-JKO achieves FID scores of 2.62 on CIFAR-10 and 5.46 on CelebA-HQ, outperforming existing WGF-based methods by a significant margin and approaching state-of-the-art performance. Our contributions can be summarized as follows: 

% To overcome these challenges, we suggest a new generative algorithm by utilizing the semi-dual form of the backward JKO scheme.
% We introduce a reparametrization trick -> scalability.
% Additionally, our paper discovers the connection between UOTM and JKO schemes; One discretization JKO step of JKO schemes is the same as solving the UOT problem.

% Through this observation, we propose a new algorithm that combines the advantages of these two methods.
% We increase the performance of the JKO scheme by relying on UOTM. 
% We also suggest efficiently parametrizing the network, hence, reducing the computation burden of the training procedure.
% Interestingly, our method shows even faster convergence than UOTM.
% Finally, by this novel parametrization, we reduce the burden on sampling, NFE 1.
% Our model shows the state-of-the-art result in transport-based generative modeling, largely exceeding the performance of previous Wasserstein gradient flow models (WGFMs) and also correcting the distributional mismatch of UOTMs.
% Overall, our contribution can be summarized as follows:
% \begin{itemize}
%     \item Identifies the relationship between JKO and UOTM
%     \item Propose a novel algorithm that is efficient in training time, and sampling time compared to prior WGFMs.
%     \item Our algorithm corrects the distributional mismatch of UOTMs, showing better convergence properties than UOTMs.
%     \item Performance very good.
% \end{itemize}

\begin{itemize}
    \item We propose a WGF-based generative model based on the semi-dual form of the JKO step.
    \item We show that the JKO step is equivalent to the Unbalanced Optimal Transport problem. This insight leads to the semi-dual form of the JKO step.
    \item Our model greatly improves the scalability of WGF models until high-dimensional image datasets. To the best of our knowledge, S-JKO is the first JKO-based generative model that presents decent performance on CelebA-HQ ($256 \times 256$).
    \item To the best of our knowledge, S-JKO is the first JKO-based generative model that achieves near state-of-the-art performance on real-world image datasets. 
\end{itemize}

\paragraph{Notations and Assumptions}
% Let $\mathcal{P}(\mathbb{R}^d)$ denote a collection of probability distributions on $\mathbb{R}^d$ that are absolutely continuous with respect to the Lebesgue measure.
% Throughout this paper, \textbf{we set $\mu=\rho_0$ and $\nu$ as the source and target distributions}, respectively. 
% In particular, $\mu$ and $\nu$ correspond to $d$-dimensional Gaussian distribution and data distribution on $\mathbb{R}^d$.
% For a measurable map $T$, $T_\# \mu$ represents the pushforward distribution of $\mu$. 
% For convenience, we set $c_h (x,y) := \frac{1}{2h} \lVert x - y \rVert_2^2$.
% Moreover, the 2-Wasserstein distance is defined as follows:
Let $\mathcal{P}(\mathbb{R}^d)$ be the set of probability distributions on $\mathbb{R}^d$ that are absolutely continuous with respect to the Lebesgue measure.
Throughout this paper, \textbf{we denote the source distribution as $\boldsymbol{\mu=\rho_0}$ and denote the target distributions as $\boldsymbol{\nu}$.}
Since our scope is on generative modeling, $\mu$ and $\nu$ correspond to the $d$-dimensional Gaussian distribution and the data distribution on $\mathbb{R}^d$, respectively. For a measurable map $T$, $T_\# \mu$ represents the pushforward distribution of $\mu$. For convenience, we set $c_h (x,y) := \frac{1}{2h} \lVert x - y \rVert_2^2$.
Moreover, the 2-Wasserstein distance $\mathcal{W}_2 (\cdot, \cdot)$ is defined as follows:
\begin{equation} \label{eq:2-was}
    \mathcal{W}_2 (\rho, \xi) := \left(\min_{\pi \in \Pi(\rho, \xi)} \int_{\mathbb{R}^d \times \mathbb{R}^d} \lVert x-y \rVert_2^2 d\pi(x,y)\right)^{\frac{1}{2}},
\end{equation}
where $\Pi(\rho, \xi)$ denotes the set of joint probability distributions on $\mathbb{R}^d\times\mathbb{R}^d$ whose marginals are $\rho$ and $\xi$. 
Moreover, $f^*$ indicates the convex conjugate of a function $f$,
\jw{i.e., $f^{*}(y) = \sup_{x \in \mathbb{R}}\{\langle x, y \rangle - f(x)\}$ for $f:\mathbb{R}\rightarrow [-\infty, \infty]$.}
% Moreover, $\Pi(\rho)$ denotes the set of joint probability density where the left marginal is $\rho$.

\section{Background} \label{sec:background}

\begin{figure*}[t]
    \captionsetup[subfigure]{aboveskip=10pt,belowskip=0pt}
    \begin{center}
        \begin{subfigure}[b]{0.7\textwidth}
            \includegraphics[width=\textwidth]{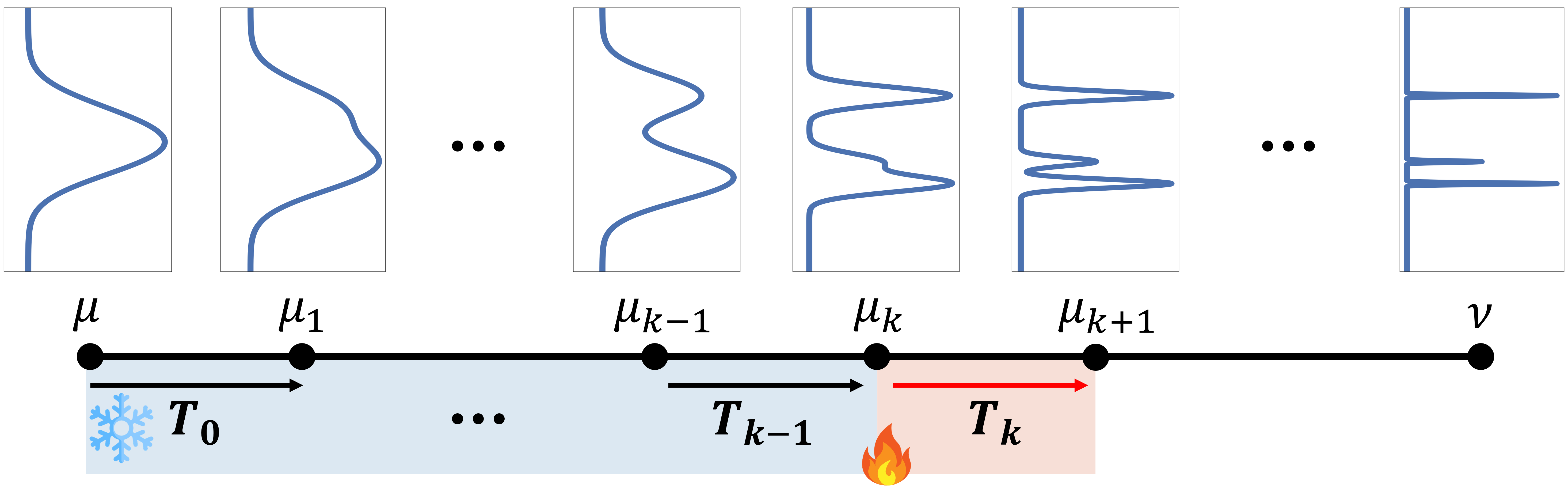}
            \caption{Training Process of Existing JKO Models}
            \label{fig:comparison_jko}
        \end{subfigure}
        \hfill
        \begin{subfigure}[b]{0.25\textwidth}        
            \includegraphics[width=\textwidth]{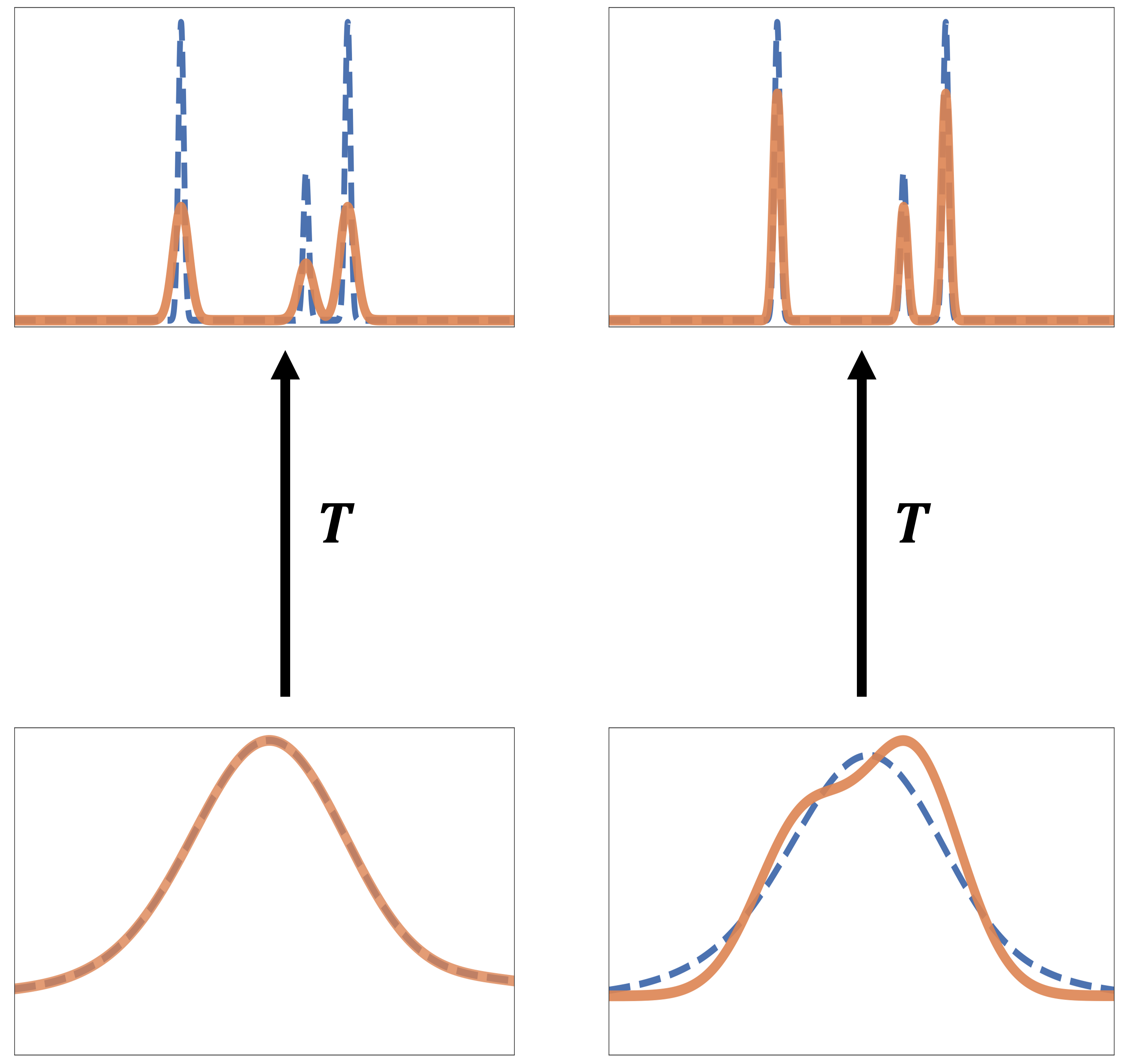}
            \caption{UOTM variants}
            \label{fig:comparison_uotm}
        \end{subfigure}
        \vspace{-5pt}
        \caption{\textbf{(a) 
        Visualization of the Training Process of Existing JKO Models.} 
        For each training iteration, sampling from $\mu_k$ involves sequential \jw{inference} through $k$-networks, i.e., $T_k \circ \dots \circ T_0(x)$ with $x\sim \mu$. This iterative network evaluation considerably slows down the training process. Formally, the training complexity becomes $O(K^2)$ where $K$ denotes the number of JKO steps.
        % For each training iteration, it is required to compute $T_k \circ \dots \circ T_0(x)$ where $x\sim \mu$ to obtain a sample from $\mu_k$. Such recursive network evaluations considerably slow down the training process.       The complexity of training time becomes $O(K^2)$ where $K$ is the number of JKO steps.
        \textbf{(b) Two Variants of the UOTMs.} \textbf{Left:} Source-fixed UOTM. \textbf{Right:} Both-relaxed-UOTM. For brevity, we simply call the Both-relaxed-UOTM as UOTM. UOTMs allow flexibility in marginal densities and therefore have inherent distribution errors (\textbf{Blue:} Source and Target distributions $\mu, \nu$. \textbf{Orange:} Marginal distributions of Optimal Coupling $\pi_0, \pi_1$.)}
        % \textbf{(b) Two variants of the UOTMs}; One is the Source-fixed UOTM (left), and the other is the both-relaxed UOTM (right). For convenience, we simply call the both-relaxed UOTM as UOTM. UOTMs relax their marginal densities, so there are inherent distribution errors (described as an orange line). }
    \end{center}
    \vspace{-15pt}
\end{figure*}

\subsection{Wasserstein Gradient Flow and JKO scheme}
\paragraph{Wasserstein Gradient Flow} 
Given a functional $\mathcal{F}(\rho)$ on $\rho\in \mathcal{P}(\mathbb{R}^d)$, the Wasserstein Gradient Flow (WGF) \cite{ambrosio} describes the dynamics of probability density $\{\rho_{t}\}_{t \geq 0}$, following the steepest descent direction of $\mathcal{F}(\rho)$. Here, the metric on $\mathcal{P}(\mathbb{R}^d)$ is defined as the 2-Wasserstein distance $\mathcal{W}_2$ (Eq \ref{eq:2-was}). 
The WGF can be explicitly written by the PDE as follows:
\begin{equation} \label{eq:WGF}
    \frac{\partial \rho}{\partial t} = \nabla \cdot \left( \rho \nabla \frac{\delta \mathcal{F}}{\delta \rho} \right),
\end{equation}
where $\frac{\delta\mathcal{F}}{\delta \rho} $ denotes the first variation of $\mathcal{F}$ with respect to standard $L_2$ metric \cite{villani}.

When $\mathcal{F}(\rho)$ is given as the $f$-divergence $D_f$ with respect to the target distribution $\nu$, \textbf{WGF describes the trajectories of probability density $\{\rho_{t}\}_{t \geq 0}$ evolving from $\mu=\rho_{0}$ towards $\nu$} by minimizing $\mathcal{F}(\rho)$:
\begin{equation} \label{eq:f-div}
    \mathcal{F}(\rho) := D_f (\rho | \nu) = \int f \left( \frac{d \rho}{d \nu} \right) d \nu.
\end{equation}
Specifically, when utilizing the KL divergence as the functional $\mathcal{F}(\rho) := D_{KL}(\rho | \nu)$, Eq \ref{eq:WGF} becomes the \textbf{Fokker-Plank equation} \jw{with the score $\nabla \log \nu$} \cite{jko}:
%with the potential $V$ \cite{jko}:
\begin{equation} \label{eq:FPK}
    \frac{\partial \rho}{\partial t} = \nabla\cdot \left( \rho \nabla \log \nu \right) + \Delta \rho, \quad \rho(0,\cdot) = \rho_0,    
\end{equation}
%where $\nu \propto e^{-V}$. 
Then, the solution $\rho_t$ converges to $\nu$ as $t\rightarrow \infty$.

\paragraph{JKO scheme}
Computing the continuous WGF is a challenging problem. To address this, \citet{jko} proposed a time discretization scheme to approximate WGF, called the \textbf{JKO scheme} (Fig \ref{fig:comparison_jko}). In this scheme, when given the current JKO step $\mu_{k}$, the next JKO step $\mu_{k+1}$ is formally defined as follows:
\begin{equation} \label{eq:jko}
    \mu_{k+1} = \underset{\rho \in \mathcal{P}(\mathbb{R}^d)}{\text{argmin}} \left[ \frac{1}{2h} \mathcal{W}_2^2 (\rho, \, \mu_{k}) + \mathcal{F}(\rho) \right].
\end{equation}
where $\mu_{0}=\mu$ is the initial condition. Intuitively, $h$ can be understood as the step size of time discretization. When the functional is set to the KL divergence $\mathcal{F}(\rho) = D_{KL}(\rho | \nu)$, the JKO scheme converges to the solution of the Fokker-Plank equation. In other words, $\{ \mu_{k} \}$ converges to $\{ \rho_{kh} \}$ in Eq \ref{eq:FPK} as the step size $h\rightarrow 0$. 

\paragraph{JKO-based Models}
In this paragraph, we provide a brief summary of previous works based on \jw{the JKO scheme \cite{wgfkorotin, jkoex1, population, vwgf, vidal2023taming, park2023deep, lee2023deep, nfjko2, riesz}}. The primary challenge in implementing the JKO scheme lies in optimizing over the probability distributions $\rho \in \mathcal{P}(\mathbb{R}^d)$. The previous works addressed this challenge by transforming it into an optimization over the transport map $T$ from $\mu_{k}$ to $\mu_{k+1}$, i.e., ${T}_\# \mu_k = \mu_{k+1}$.
Note that
\begin{equation} \label{eq:monge}
   \mathcal{W}_2^2 (\mu_k, \mu_{k+1}) = \min_{{T}_\# \mu_k = \mu_{k+1}} \int_{\mathbb{R}^d} \lVert x-T(x) \rVert_2^2 d\mu_k(x),
\end{equation}
where $T$ is a measurable map.
The transport map $T$ that minimizes Eq \ref{eq:monge} is referred to as the optimal transport map from $\mu_k$ to $\mu_{k+1}$.
Using this fact, \citet{vwgf} reparametrizes the JKO step (Eq \ref{eq:jko}) as follows:
\begin{align}
    \begin{split}
        &\mu_{k+1} = {T_{k}}_\# \mu_k, \\
        &T_k = \underset{T}{\text{argmin}} \ \frac{1}{2h} \int_{\mathbb{R}^d} \lVert x - T(x) \rVert_2^2 d\mu_k(x) + \mathcal{F}\left( T_\# \mu_k \right).
    \end{split}
\end{align}
% Furthermore, by Brenier's theorem, there exists a convex function $\psi$ such that the optimal transport map $T_k$ is a gradient of $\psi$.
Moreover, Brenier's theorem states that there exists a convex function $\psi$ such that the optimal transport map $T_k$ is a gradient of $\psi$, i.e., $T_k=\nabla \psi$. 
Leveraging this fact, several works \cite{wgfkorotin,jkoex1, population} parameterizes $T$ as the gradient of input convex neural network (ICNN) \cite{icnn}.
\jw{However, these JKO-based models suffered from the quadratic complexity $O(K^2)$, where $K$ denotes the number of JKO steps. In this regard, \citet{SWGF} suggested the Sliced-Wasserstein Gradient Flow to mitigate this complexity to $O(K)$.}

\begin{figure*}[t]
    \centering
        \includegraphics[width=.7\textwidth]{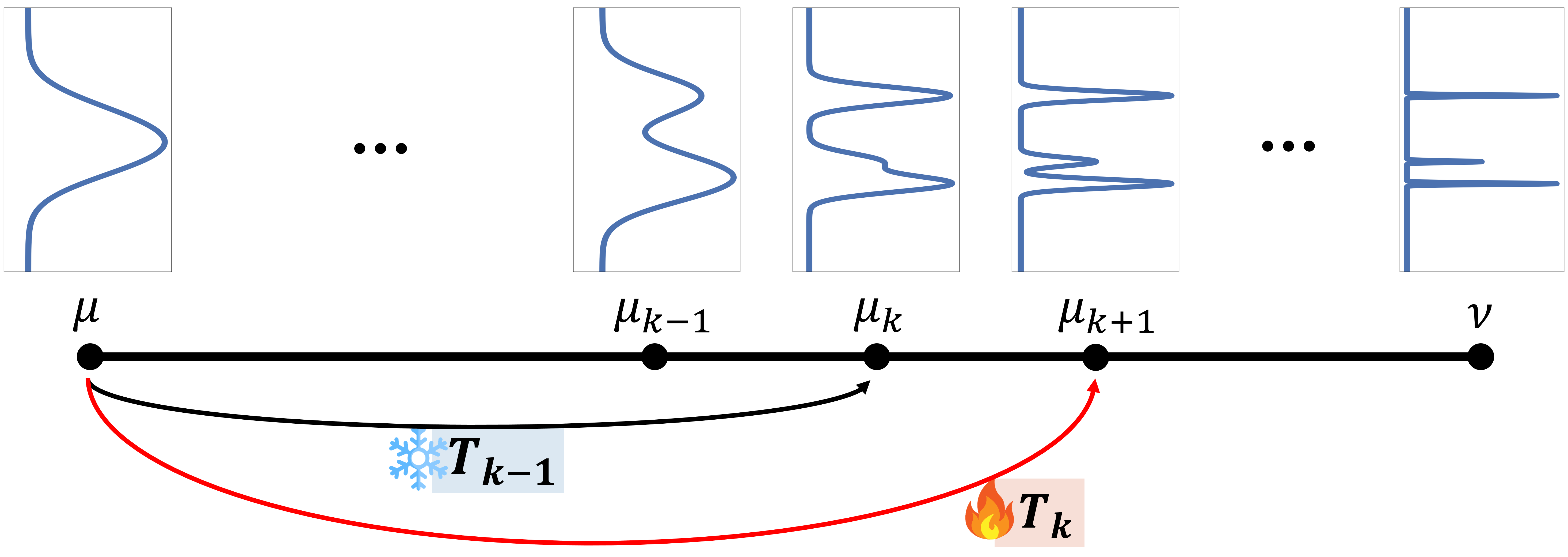}
        \vspace{-5pt}
        \caption{\textbf{Conceptual Diagram of Our Model.} 
        During the training $k$-th JKO step in our model, sampling from $\mu_k$ involves only one network inference $T_{k-1}$, i.e., $\mu_k = {T_{k-1}}_\# \mu_0$. This reparametrization strategy significantly reduces the overall training time. Formally, the training time complexity reduces to $O(K)$ from the $O(K^2)$ of other JKO models. Moreover, by initializing the parameters of $T_k$ with $T_{k-1}$, we can further decrease the number of iterations required for training. 
        % During training $k$th JKO step, our model generates samples from $\mu_k$ within a single network evaluation of $T_{k-1}$, i.e. $\mu_k = {T_{k-1}}_\# \mu_0$. This reparametrization strategy significantly reduces overall training time compared to other JKO schemes because they generate samples of $\mu_k$ by recursive network evaluations (Fig \ref{fig:comparison_jko}). Formally, the training time complexity reduces to $O(K)$ where $K$ is the number of JKO steps. Additionally, by initializing parameters of $T_k$ by $T_{k-1}$, we can further reduce the number of iterations required for the training. 
        % Moreover, since our approach does not involve repeated network evaluations within each iteration, we can comfortably utilize a large network architecture without encountering any concerns related to training time and memory usage.
        % Moreover, due to the single-step evaluation, we can comfortably utilize a large network architecture without encountering any concerns related to training time and memory usage.
        \label{fig:concept_ours}
        \vspace{-23pt}
}
    \hfill
    % \begin{minipage}{0.3\textwidth}
    %     \centering
    %     \caption{Comparison of various JKO schemes on CIFAR10. Time and Complexity are a wall-clock training time and complexity of training time with respect to $K$, $d$, and $T$. Here, $T$ is the number of operations required to evaluate the network $T_\theta$. NFE is the number of function evaluations required to generate a sample. Note that training time is measured on 1 GPU (RTX 3090 Ti).}
    %     \vspace{10pt}
    %     \scalebox{0.85}{
    %     \begin{tabular}{ccccc}
    %         \toprule
    %         Model   &  Time & Complexity &  NFE & FID \\
    %         \midrule
    %         \citet{wgfkorotin}  &  - & $O(K^2d^3 T)$ & - & -  \\
    %         \citet{vwgf} &     $\geq$ 50h & $O(K^2 T)$ & 160 & 23.1 \\
    %         \citet{nfjko}  & $\geq$ 30h  & $O(K^2 T)$ & $\geq$ 30 & 29.1  \\
    %         \midrule
    %         Ours (\textit{small})  & 6h & $O(K T)$ & 1 & 8.78 \\
    %         Ours (\textit{large})  & 50h & $O(K T)$ & 1 & 2.65 \\
    %         \bottomrule
    %     \end{tabular}}
    % \end{minipage}
\end{figure*}

\subsection{Optimal Transport-based Generative Modeling} \label{sec:background_uot}
\paragraph{Unbalanced Optimal Transport (UOT)}
The classical OT problem investigates the cost-minimizing transport map that satisfies an exact matching between two distributions \cite{villani}.
Recently, a new variation of the OT problem has
been introduced, which is called \textbf{Unbalanced Optimal Transport (UOT)} \cite{uot1, uot2}. Formally, the UOT problem between the source distribution $\mu$ and the target distribution $\nu$ is defined as follows:
\begin{equation} \label{eq:uot}
    \inf_{\pi \in \mathcal{M}_+} \int_{\mathcal{X}\times \mathcal{Y}} c(x,y) d\pi(x,y)  + D_{\varphi_1}(\pi_0|\mu) + D_{\varphi_2}(\pi_1|\nu),
\end{equation}
where $\mathcal{M}_+$ denotes a set of positive Radon measures on $\mathcal{X}\times \mathcal{Y} = \mathbb{R}^d\times \mathbb{R}^d$ and $c(\cdot, \cdot)$ represents the transportation cost. $D_{\varphi_1}$, $D_{\varphi_2}$ are two $f$-divergence\footnote{In the general case, $D_{\varphi}$ is defined as the \csiszar{} divergence \cite{csiszar} $D_{\varphi}^{c}$ for the UOT problem. Note that when $\mu$ is absolutely continuous with respect to $\nu$, the $f$-divergence and the \csiszar{} divergence are equivalent, i.e., $D_{\varphi}^{c}(\mu | \nu) = D_{\varphi}(\mu | \nu)$.} terms that penalize the dissimilarity between the marginal distributions $\pi_{0}, \pi_{1}$ and $\mu,\nu$, respectively. \textbf{Note that, in the UOT problem, both marginals of $\pi$ are not explicitly fixed.} Instead, these marginals are softly regularized by the divergence terms. This flexibility in the marginals provides outlier robustness \cite{robust-ot}.

Furthermore, the UOT problem is a generalization of the classical OT problem. When we choose $\varphi_1$ or $\varphi_2$ to be a convex indicator function $\iota(x) = \begin{cases} 0 & \text{if } x=1, \\  \infty & \text{otherwise} \end{cases}$,
% % at the set $\{1\}$, 
% \begin{equation} \label{eq:indicator}
%     \iota(x) = 
%     \begin{cases} 
% 		0 & \text{if } x=0. \\ 
%         \infty & \text{otherwise. }
%      \end{cases}    
% \end{equation}
then $D_{\varphi}$ takes the following form:
\begin{equation}
    D_{\iota}(\pi_{i}|\rho) = 
    \begin{cases} 
		0 & \text{if } \pi_{i} = \rho \,\,\text{ almost-surely.} \\ 
        \infty & \text{otherwise. }
     \end{cases}    
\end{equation}
Therefore, setting $\varphi_1=\iota$ or $\varphi_2=\iota$ means fixing the source distribution, i.e., $\pi_{0}=\mu$, or the target distribution, i.e., $\pi_{1}=\nu$. When we fix both distributions, the UOT problem is simplified to the OT problem. 
Throughout this paper, we refer to the UOT problem when $\varphi_1=\iota$ as the \textbf{Source-fixed UOT problem}.
% \jw{Throughout this paper, we call the problem is \textbf{Source-fixed UOT problem} when $\varphi_1=\iota$ in Eq \ref{eq:uot}.
This problem serves an important role in our work in Sec \ref{sec:method}.
% The \textbf{Source-fixed variant of the UOT problem}, where $\pi_{0}=\mu$, 

\paragraph{UOT-based generative models}
Recently, \citet{uotm} proposed a class of generative models by leveraging the semi-dual form of the UOT problem.
Formally, the semi-dual form of Eq \ref{eq:uot} is defined as follows:
\begin{multline} \label{eq:OTGAN}
    \sup_{v \in \mathcal{C}} \int \varphi_1^\circ \left(\inf_T \left[ c\left(x, T(x)\right) - v(T(x)) \right]\right) d\mu(x) \\ 
    + \int \varphi_2^\circ (v(y)) d\nu(y),
\end{multline}
where $\mathcal{C}$ denotes a set of continuous functions over $\mathbb{R}^d$ and $\varphi_i^\circ(x) := - \varphi_i^* (-x)$.
By parametrizing $v=v_\phi$ and $T=T_\theta$ in Eq \ref{eq:OTGAN} by neural networks, \citet{uotm} suggested the max-min adversarial learning objective, called UOTM. In this framework, $T_\theta$ represents the (unbalanced) transport map from $\mu$ to $\nu$, and $v_{\phi}$ serves as the potential \jw{function} for discriminating between  $T_{\# }\mu$ and $\nu$. 
\citet{uotmsd} demonstrated that the flexibility in distribution matching enhances the stability of the training process. However, \textbf{this flexibility also introduces inherent distribution errors to UOTM} \cite{uotm}.
Throughout this paper, we call the UOTM variant corresponding to the Source-fixed UOT problem as the \textbf{Source-fixed-UOTM}. Specifically, this is equivalent to choosing $D_{\varphi_1}=D_{\iota}$ and, thereby, $\varphi_i^\circ(x)=x$. On the contrary, when both $\varphi_1$ and $\varphi_2$ in Eq \ref{eq:uot} are not convex indicators, we denote it as \textbf{UOTM} or \textbf{Both-relaxed-UOTM} (Fig \ref{fig:comparison_uotm}).

\begin{table}[t]
    \centering
    \caption{
    \textbf{Scalability Comparison for Various JKO Schemes on CIFAR-10.} Time denotes a wall-clock training time. Complexity indicates the training complexity with respect to $K$ and $d$. \jw{Here, we only consider the complexity of algorithms, not the complexity of backbone network inference.} NFE refers to the number of function evaluations required to produce a sample. Note that training time is measured on 1 GPU (RTX 3090 Ti).
    % Comparison of various JKO schemes on CIFAR10. Time is a wall-clock training time. Complexity is the complexity of training time with respect to $K$, $d$ and $T$. Here, $T$ is the number of operations required to evaluate the network $T_\theta$. NFE is the number of function evaluations required to generate a sample. Note that training time is measured on 1 GPU (RTX 3090 Ti).
        }
    \vspace{10pt}
    \scalebox{0.72}{
        \begin{tabular}{ccccc}
            \toprule
            Model   &  Time  & Complexity &  NFE ($\downarrow$) & FID ($\downarrow$) \\
            \midrule
            \citet{wgfkorotin}  &  - & $O(K^2d^3 )$ & - & -  \\
            \citet{vwgf} &     $\geq$ 50h & $O(K^2 )$ & 160 & 23.1 \\
            \citet{nfjko}  & $\geq$ 30h  & $O(K^2 )$ & $\geq$ 150 & 29.1  \\
            \midrule
            Source-fixed-UOTM (\textit{Small}) & \textbf{6h} & $O(1)$ & 1 & 14.4 \\
            Ours (\textit{Small})  & \textbf{6h} & $O(K )$ & 1 & \textbf{8.78} \\
            \midrule
            Source-fixed-UOTM  (\textit{Large}) & 50h & $O(1)$ & 1 & 7.53 \\            
            Ours (\textit{Large})  & 50h & $O(K )$ & 1 & \textbf{2.65} \\
            \bottomrule
        \end{tabular}}
    \label{tab:JKO-comparison}
    \vspace{-10pt}
\end{table}

\section{Limited Scalability of WGF Models} \label{sec:limWGF}
% In this section, we investigate the limitations of two approaches for connecting two probability distributions in generative modeling: Scalability for Wasserstein Gradient Flow (WGF) models and Distribution Error for Unbalanced Optimal Transport Map (UOTM). In Sec \ref{sec:method}, we will introduce \textbf{a scalable WGF-based generative model that effectively addresses both challenges by exploiting the equivalence between the two approaches}.

% \paragraph{Training and Evaluation Burden of WGFMs}
% \subsection{Limited Scalability of WGFs} \label{sec:limWGF}
\paragraph{Quadratic Complexity of JKO models}
The primary challenge for WGF models lies in their limited scalability when dealing with \jw{complex} high-dimensional image datasets.
This limitation stems from the iterative multi-step approximation of intermediate distributions $\rho_{t}$ in WGF. 
\textbf{This iterative approximation considerably slows down the training process quadratically, i.e., $O(K^2)$, with respect to the total number of approximation steps $K$} (Table \ref{tab:JKO-comparison}). 
Consequently, the scalability of WGF models is significantly constrained.
Specifically, as described in Sec \ref{sec:background}, \jw{most WGF models employ the JKO scheme to numerically approximate WGF \cite{wgfkorotin, jkoex1, population, vwgf, vidal2023taming, park2023deep, lee2023deep, nfjko2, riesz}.}
%Note that the JKO scheme is the time discretization of WGF.
Hence, these models involve the iterative estimation of $(k+1)$-th distribution $\mu_{k+1}$, based on the $k$-th distribution $\mu_{k}$ (Eq \ref{eq:jko}). \textbf{Each estimation requires a neural network training for learning each transport map $T_{k}$ from $\mu_{k}$ to $\mu_{k+1}$}, i.e., $(T_{k})_{\#} \mu_{k} = \mu_{k+1}$ (Fig \ref{fig:comparison_jko}). Note that, for each $T_{k}$,\textbf{ the source data sampling from $\mu_{k}$ requires inference from all $\{T_{i}\}_{0 \leq i \leq k-1}$ with $x \sim \mu$}:
\begin{equation} \label{eq:jko-simul}
    x_{k} \sim \rho_{k} \quad  \Leftrightarrow \quad  \left( T_{k-1}\circ T_{k-2} \circ \dots \circ T_0\right) (x).
    % x_{i+1} = T_{i}(x_{i}) \,\text{ for }\, 0 \leq i \leq k-1.
\end{equation}

\paragraph{Comparison to Ours}
Table \ref{tab:JKO-comparison} presents a comparison of the scalability (in terms of training time and complexity) of various JKO models on CIFAR-10. We compared the complexity of previous works, our JKO-based model, and the UOTM-based counterpart of our model (Source-fixed-UOTM).
Note that the \textit{Small} backbone network for our model presents a comparable size to previous JKO models (See the Appendix \ref{appen:Implementation} for details). 
Therefore, \textbf{in terms of scalability, this section focuses on the comparison between the prior JKO models and our models with the \textit{Small} backbone}.\footnote{The experimental results using a \textit{Large} backbone demonstrate that our model is scalable to the competitive backbone network (NCSN++), which is widely employed by state-of-the-art generative models, and can provide comparable performance with it.  A more comprehensive discussion will be provided in Sec \ref{sec:exp}.}

The prior JKO models typically utilized $K \geq 150$ JKO steps for approximating WGF on CIFAR-10 \cite{vwgf, nfjko}, and $K=50 \sim 150$ JKO steps on low-dimensional ($\sim 100$ dimensions) datasets \cite{vwgf, wgfkorotin}. In other words, \textit{each WGF model consisted of 50-150 small neural networks, with each neural network dedicated to approximating $T_{k}$}. In this respect, the quadratic training complexity $O(K^{2})$ considerably limited the scalability of WGF models, by restricting the \jw{size} of each neural network. As a result, when extended to high-dimensional image datasets of CIFAR-10, WFG models are typically adapted to op3erate on the latent space of encoder-decoder architectures \cite{vwgf, nfjko}. Nevertheless, these models suffer from long training time ($\geq 30h$) and non-competitive generation results of FID score ($\geq 20$) (Table \ref{tab:JKO-comparison}).
In this paper, we significantly improve the scalability of WGF models by discovering that the JKO step can be interpreted as the Unbalanced Optimal Transport problem (Sec \ref{sec:relationship}). 
Compared to existing WFG models of similar size, our model outperforms them with a lower FID score of $8.78$ on CIFAR-10, while requiring a much less training time of 6 hours.

\section{Method} \label{sec:method}
In this section, we propose a novel WGF-based generative model, called the Semi-dual JKO scheme (\textbf{\textit{S-JKO}}). Our model is based on the equivalence between the JKO step and the Unbalanced Optimal Transport problem (Sec \ref{sec:relationship}).  Building upon this insight, we introduce a generative model based on the semi-dual form of the JKO step (Sec \ref{sec:algorithm}).

% \subsection{Proposed Method}
\subsection{Equivalence between JKO step and UOT problem} \label{sec:relationship}
In this subsection, we establish the equivalence between the JKO step (Eq \ref{eq:jko}) and the Source-fixed variant of the Unbalanced Optimal Transport problem (Eq \ref{eq:uot}).  
Here, we begin with the JKO step.
Let $\mu = \mu_0$ and $\nu$ denote the source and target distributions, respectively. \jw{As a reminder, our primary focus} is generative modeling. Hence, $\mu$ represents the prior distribution (Gaussian), and $\nu$ corresponds to the target data distribution. We define the energy functional $F(\rho)$ associated with the JKO step as $\mathcal{F}(\rho) = D_f \left( \rho | \nu \right)$. Then, the JKO step (Eq \ref{eq:jko}) can be expressed as follows:
\begin{equation} \label{eq:jko-again}
    \mu_{k+1} = \underset{\rho \in \mathcal{P}(\mathbb{R}^d)}{\text{argmin}} \underbrace{\frac{1}{2h} \mathcal{W}_2^2 (\mu_{k}, \rho) + D_f (\rho | \nu)}_{\mathcal{L}_{\rho}}.
\end{equation}
If we expand  $\mathcal{W}_2^2 (\mu_{k}, \rho)$ using its definition (Eq \ref{eq:2-was}), then $\mathcal{L}_{\rho}$ can be rewritten as the follows (See the appendix for details):
% \begin{multline}
%     \mathcal{L}_{\rho} =   \left(\min_{\pi \in \Pi(\mu_{k}, \rho)} \int \frac{1}{2h} \lVert x-y \rVert_2^2 d\pi(x,y) \right) \\
%     + D_f (\rho | \nu).
% \end{multline}
\begin{equation}
    \mathcal{L}_{\rho} =   \min_{\pi \in \Pi(\mu_{k}, \rho)} \int \frac{1}{2h} \lVert x-y \rVert_2^2 d\pi(x,y)  + D_f (\pi_{1} | \nu).
\end{equation}
Note that $\pi_1 = \rho$ in the above equation.
% Because $\pi_{1}=\rho$ and $D_f (\pi_{1} | \nu)$ is constant with respect to $\pi \in \Pi(\mu_{k}, \rho)$,
% \begin{equation} 
%     \mathcal{L}_{\rho} =   \min_{\pi \in \Pi(\mu_{k}, \rho)} \int \frac{1}{2h} \lVert x-y \rVert_2^2 d\pi(x,y)  + D_f (\pi_{1} | \nu).
% \end{equation}
Therefore, when combined with the minimization over $\rho$ in Eq \ref{eq:jko-again}, the JKO step is equivalent to the Source-fixed UOT problem, i.e., convex indicator $\varphi_{1}=\iota$ and $\varphi_{2}=f$ in Eq \ref{eq:uot}:
\begin{align} 
    &\pi^{\star} =  \underset{\pi_{0}=\mu_{k}}{\text{argmin}}
    \int c_{h}(x,y) d\pi(x,y) + D_{f}(\pi_1|\nu). \label{eq:equiv} \\
    &\mu_{k+1} =\pi^{\star}_{1}.
\end{align}
with $c_{h}(x,y) = \frac{1}{2h} \lVert x-y \rVert_2^2$.
Note that \jw{in the UOT problem, when} $\mu$, $\nu$ are probability distributions (i.e., positive measures with a total mass of 1), then the optimal $\pi^{\star}$ also has the same total mass \cite{semi-dual3}. Therefore, performing the optimization over the positive Radon measure in Eq \ref{eq:uot} is equivalent to performing the optimization over the joint probability distribution in Eq \ref{eq:equiv}.

% If we set $J(\rho) = D_f (\rho| \rho^\star)$, one-step of JKO scheme boils down into target-relaxed-UOTM.
% Example KL, Chi-square, Softplus, JSD.

% Suppose $\mu = \mu_0$ and $\nu$ are a source and target distribution, respectively. 
% Let $\mathcal{F}(\rho) := D_f \left( \rho | \nu \right)$.
% Then, our optimization problem Eq \ref{eq:jko} is as the follows: 
% \begin{equation} \label{eq:jko-again}
%     \mu_{k+1} = \underset{\rho \in \mathcal{P}(\mathbb{R}^d)}{\text{argmin}} \left[ \frac{1}{2h} \mathcal{W}_2^2 (\mu_{k}, \rho) + D_f (\rho | \nu) \right].
% \end{equation}
% In other words, by the definition of 2-Wasserstein distance $\mathcal{W}_2$, the objective function of the optimization problem is to find a positive joint Radon measure $\pi \in \mathcal{M}_+$ with restriction $\pi_0 = \mu_k$, minimizing the following cost term:  $\int \frac{1}{2h}  \lVert x-y \rVert_2^2 d\pi(x,y) + D_f (\pi_1 | \nu).$
% Formally, Eq \ref{eq:jko-again} can be reformulated as the follows:
% \begin{equation}
%     \underset{\pi_0 = \mu_k}{\inf} \left[ \int_{\mathbb{R}^d\times \mathbb{R}^d} \frac{1}{2h} \lVert x-y \rVert_2^2 d\pi(x,y) +D_f (\pi_1 | \nu)   \right].
% \end{equation}

\subsection{Generative Modeling with the Semi-dual Form of JKO step} \label{sec:algorithm}
In this subsection, we propose a generative model based on the JKO scheme for the WGF. Our model is derived through two steps: (1) Semi-dual form of the JKO scheme from the equivalence with the UOT problem and (2) Reparametrization trick for enhancing the scalability of the JKO scheme.

\paragraph{Semi-dual form of JKO step}
% The semi-dual form of the JKO step is obtained through the semi-dual form of the UOT problem \cite{uotm}. By selecting  $\varphi_1 = \iota$ and $\varphi_2 = f$ in Eq \ref{eq:uot}, we can express the semi-dual form of Eq \ref{eq:equiv} as follows (See the appendix for detail):
\jw{The semi-dual form of JKO step is obtained from its UOT interpretation (Eq \ref{eq:equiv}). By setting $\varphi_1 = \iota$ and $\varphi_2 = f$ in Eq \ref{eq:uot}, we can derive the semi-dual form of JKO step from the semi-dual form of UOT \cite{uotm} as follows (See the appendix for detail):
}
\begin{equation} 
    \sup_{v \in \mathcal{C}} \int v^c(x) d\mu_k(x) + \int f^\circ(v(y)) d\nu(y).
\end{equation}
where the $c$-transform of $v$ is defined as $v^c(x) := \inf_y ( c(x,y) - v(y) )$. Here, we parametrize $\Delta T_{k}$ as follows \cite{otm}:
\begin{equation}
    \Delta T_{k} : x \mapsto \underset{y}{\arg\min} \left( c(x,y) - v(y) \right).
\end{equation}
Then, $\Delta T_{k}$ satisfies the following:
\begin{equation}
    v^c(x) := c(x,\Delta T_{k}(x)) - v(\Delta T_{k}(x)).
\end{equation}
Therefore, the semi-dual form of the JKO step can be represented as the following adversarial learning objective:
\begin{multline} \label{eq:semi-dual-jko}
    \sup_{v \in \mathcal{C}} \int \inf_{\Delta T_{k}} \left[ c\left(x, \Delta T_{k}(x)\right) - v(\Delta T_{k}(x)) \right] d\mu_k(x) \\ 
    + \int f^\circ(v(y)) d\nu(y).
\end{multline}
Note that this objective for a single JKO step is equivalent to Source-fixed-UOTM \cite{uotm} between $\mu_{k}$ and $\mu_{k+1}= (\Delta T_{k})_{\#} \mu_{k}$. 
% \todo{Why the semi-dual is needed? What happens it we introduce the reparametrization trick to primal form?}

\begin{algorithm}[t]
\caption{Training algorithm}
    \begin{algorithmic}[1]
    % \Require 
    \REQUIRE Transport network $T_\theta$ and the discriminator network $v_\phi$. 
    \STATE $T_{\text{old}} = \text{Id}$
    \FOR{$k = 0, 1, 2 , \dots, K$}
        \FOR{$i = 0, 1, 2 , \dots, N$}
            \STATE Sample a batch $x\sim \mu$ and $y\sim \nu$.
            \STATE $\hat{y} = T_\theta(x)$.
            \STATE Update $\phi$ by \jw{minimizing the objective} $\mathcal{L}_v$.\\
            \vspace{-10pt}
            $$\mathcal{L}_v = v_\phi (\hat{y}) - f^\circ\left(v_\phi (y)\right)$$
            \vspace{-15pt}
            \STATE Sample a batch $x\sim \mu$.
            \STATE $\hat{y} = T_\theta(x)$, $\hat{y}_{\text{old}} = T_{\text{old}}(x)$.
            \STATE Update $\theta$ by \jw{minimizing the objective} $\mathcal{L}_T$. \\
            \vspace{-10pt}
            $$ \mathcal{L}_T = c\left(\hat{y}_{\text{old}}, \hat{y}\right) - v_\phi(\hat{y}) $$
            \vspace{-15pt}
        \ENDFOR
        \STATE $T_{\text{old}} \leftarrow T_\theta$
    \ENDFOR
    \end{algorithmic}
\label{alg:sjko}
% \vspace{-10pt}
\end{algorithm}

\begin{figure*}[h] 
    \centering
    \begin{subfigure}[b]{0.49\textwidth}
        \includegraphics[width=\textwidth]{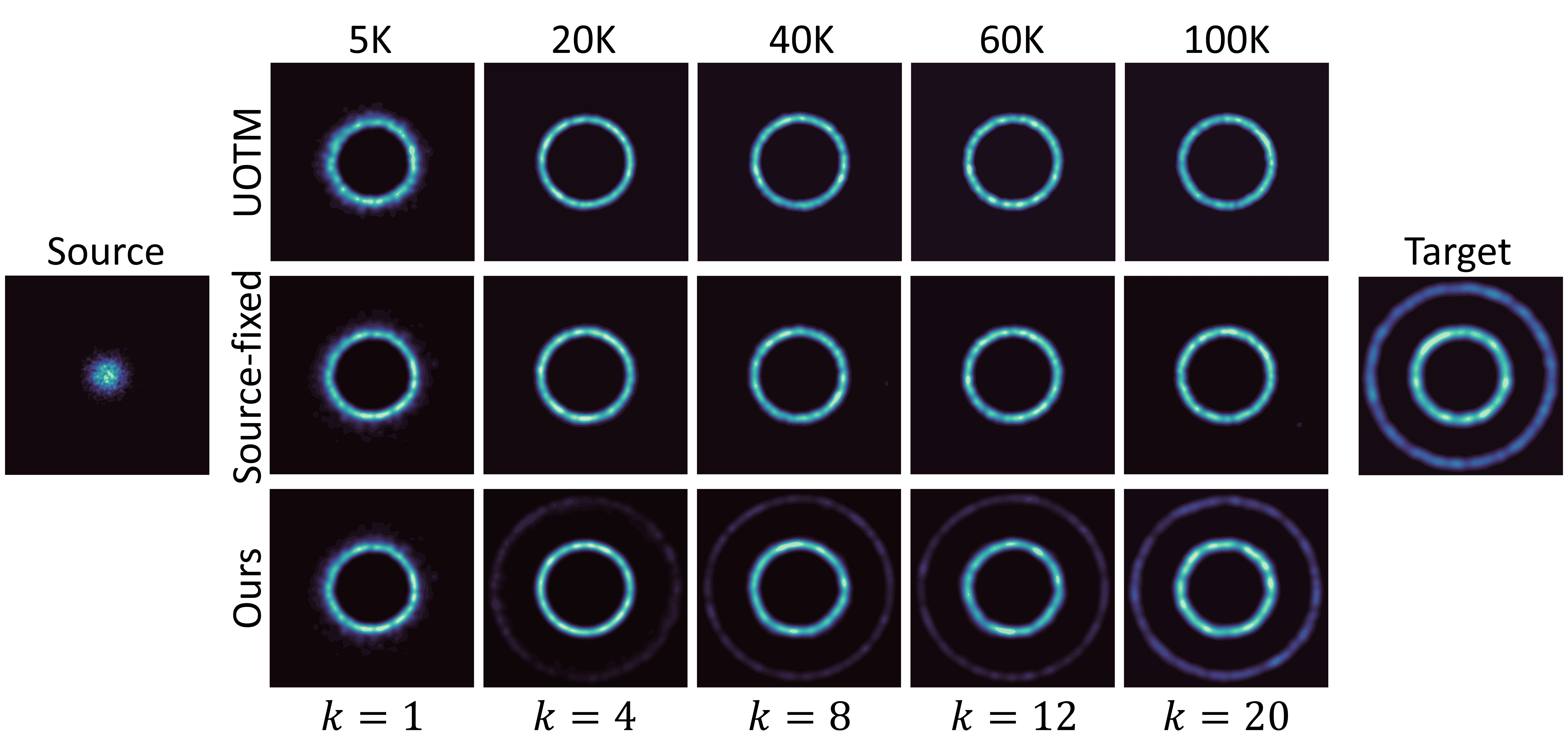}
        \caption{Two Circles}
    \end{subfigure}
    \hfill
    \begin{subfigure}[b]{0.49\textwidth}
        \includegraphics[width=\textwidth]{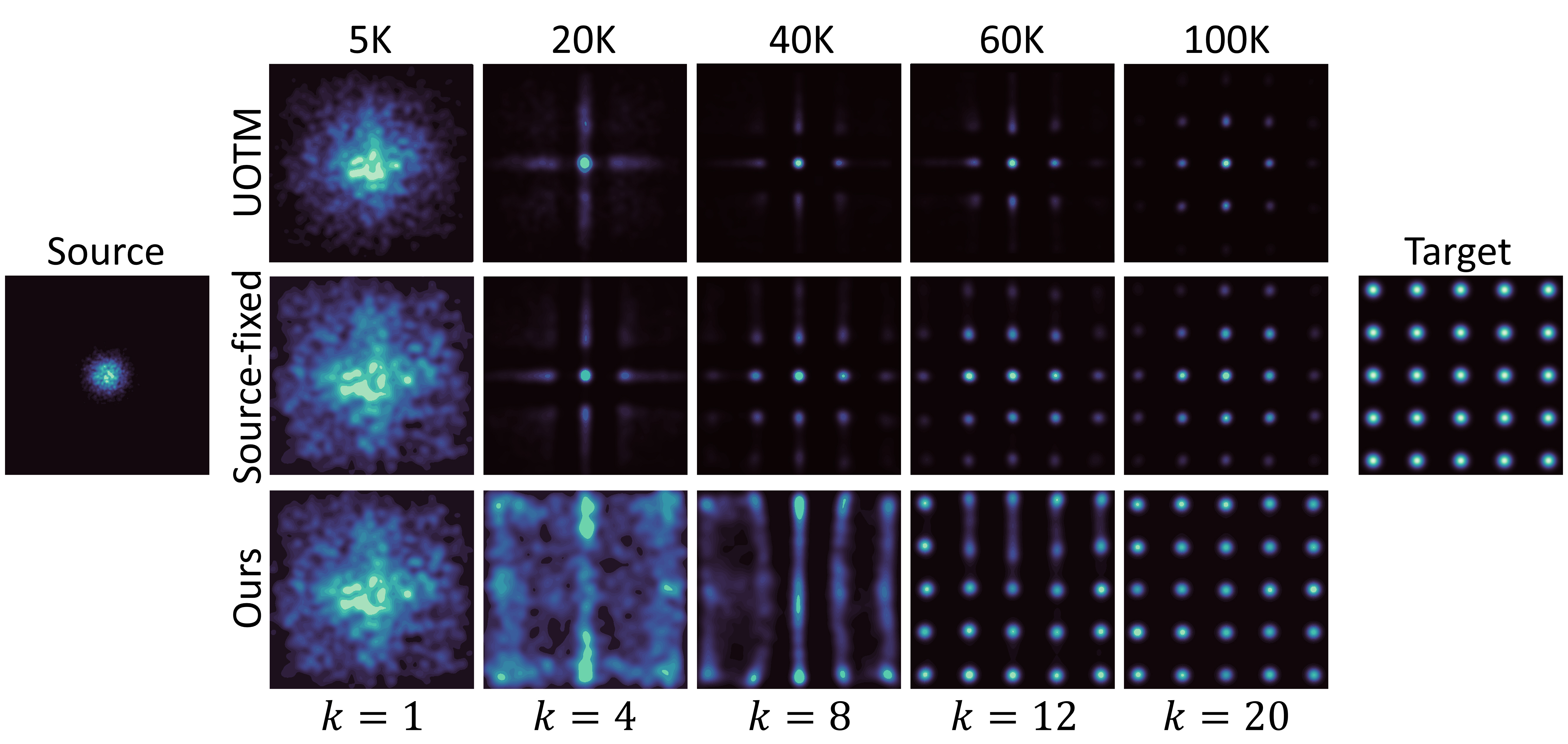}
        \caption{25-Gaussian Mixture}
    \end{subfigure}
    \vspace{-10pt}
    \caption{\textbf{Generation Results for UOTM, \jw{Source}-fixed UOTM, and S-JKO on Synthetic Datasets}. Each column shows the generated distribution at training iterations \{5, 20, 40, 60, 100\}K. $k$ denotes the index of JKO steps corresponding to that particular iteration.
    % Visualization results for UOTM, Sorce-fixed UOTM, and S-JKO on synthetic dataset. The results at the training iteration of 5K, 20K, 40K, 60K, and 100K are illustrated. For each iteration, the number of JKO steps for S-JKO is 1, 4, 8, 12, and 20.
    }
    \label{fig:toys}
    \vspace{-8pt}
\end{figure*}

\paragraph{Reparametrization trick}
The quadratic training complexity of the JKO models (Table \ref{tab:JKO-comparison}) stems from the necessity to simulate the entire trajectory of the JKO scheme $\{ \mu_{k} \}_{k}$ (Fig \ref{fig:comparison_jko}). To manage this challenge, we introduce a straightforward reparametrization trick. Suppose $T_k:\mathbb{R}^d \rightarrow \mathbb{R}^d$ be a measurable map such that
\begin{equation}
    T_k = \Delta T_{k}  \circ \cdots \circ \Delta T_{1} \circ \Delta T_{0}.
\end{equation}
$T_k$ satisfies $T_{k} := \Delta T_{k} \circ T_{k-1}$ and $(T_k)_\# \mu = \mu_{k+1}$.
Now, we introduce the reparametrization trick to $\mu_{k}$ in Eq \ref{eq:semi-dual-jko}:
\begin{multline} \label{eq:reparametrize}
    \mathcal{L}_{k} = \sup_{v \in \mathcal{C}} \int \inf_{T_{k}} [ c\left(T_{k-1} (x),  \, T_k (x)\right) - v(T_k(x))] d\mu(x) \\
    + \int f^\circ(v(y)) d\nu(y).
\end{multline}
Note that this reparametrized learning objective $\mathcal{L}_{k}$ is for the $k$-th JKO step. The comprehensive training procedure repeats this step for $k=0, \cdots, K$. Our reparametrization trick transforms the learning objective from $\Delta T_{k}$-training to $T_{k}$-training. Therefore, for each phase $k$, we possess a direct transport map $T_{k-1}$ that connects $\mu$ to $\mu_{k}$. In other words, for each phase, sampling from $\mu_{k}$ does not require simulating the entire trajectory. Instead, we can efficiently generate $\mu_{k}$ using a one-step inference. In this regard, \textbf{this reparametrization trick significantly contributes to managing the training complexity (Table \ref{tab:JKO-comparison}).}
Furthermore, there is an additional advantage in terms of training efficiency. For each phase transition, we initialize $T_{k}$ using the previous $T_{k-1}$. We hypothesize that this contributes to stable training during the entire training process 
% compared to random initialization of $\Delta T_{k}$ for each $k$. 
This training efficiency is empirically demonstrated in Sec \ref{sec:exp}.

\paragraph{Algorithm}
Finally, we present our training algorithm (Algorithm \ref{alg:sjko}), called the Semi-dual JKO scheme (\textbf{S-JKO}). The adversarial learning objective $\mathcal{L}_{k}$ is updated through alternating gradient descent, as in GAN \cite{gan}. We simplified Algorithm \ref{alg:sjko} by excluding non-dependent terms for each $v_{\phi}$ and $T_{\theta}$ (See the appendix for the more detailed Algorithm). Additionally, \textbf{note that when we conduct training for only one phase, i.e., $K=1$, our S-JKO is equivalent to the Source-fixed-UOTM \cite{uotm}}. In Sec 5, we will provide further clarification regarding the advantages over Source-fixed-UOTM.

% \section{Related Works}
% \paragraph{JKO scheme}

% \paragraph{OT-based Generative Models}
\section{Experiments} \label{sec:exp}
In this section, we conduct experiments on the various datasets to evaluate the following aspects of our model:
% In this section, we present qualitative and quantitative generation results of our model.
% We consider 
% Our discussion is organized as follows:
\begin{itemize}[topsep=-1.5pt, itemsep=0pt]
    \item In Sec. \ref{sec:synthetic}, we compare S-JKO with UOTMs regarding the distribution error between the generated and target distributions on synthetic datasets. % Distribution Matching on synthetic datasets.
    % In Sec. \ref{sec:synthetic}, we compare our model with UOTMs on synthetic datasets to verify whether each method accurately transports one distribution to another. Moreover, we visualize the training dynamics of our model.
    \item \jw{In Sec. \ref{sec:imagegeneration}}, we compare S-JKO with other JKO models regarding scalability on large-scale image datasets. Moreover, we demonstrate that S-JKO achieves competitive performance compared to state-of-the-art generative models.
    % In Sec. \ref{sec:imagegeneration}, we mainly compare the scalability of our model with JKO schemes on large-scale image datasets. Moreover, we verify whether our model shows descent performance compared to leading generative models. In particular, we evaluate on CIFAR-10 \cite{cifar10} and CelebA-HQ \cite{celeba} ($256\times 256$) datasets. 
    \item In Sec. \ref{sec:ablation}, we assess the robustness of S-JKO regarding JKO hyperparameters through ablation studies, such as the step size $h$, the number of JKO steps $K$, and the functional $\mathcal{F}(\cdot)$.
    % In Sec. \ref{sec:ablation}, to verify the robustness of our model against JKO hyperparameters, we conduct ablation studies on the JKO step size $h$ and the number of JKO steps $K$.
\end{itemize}
Throughout this paper, we considered two functionals for $\mathcal{F}(\cdot)$: KL divergence (KLD) $\mathcal{F}(\cdot)=D_{KL}(\cdot|\nu)$ and Jensen-Shannon divergence (JSD) $\mathcal{F}(\cdot)=D_{JSD}(\cdot|\nu)$.
Unless otherwise stated, $\mathcal{F}(\cdot)$ is KLD. For further implementation details, please refer to Appendix \ref{appen:Implementation}. 
% Throughout this paper, $D_f$ is set as a KL divergence (KLD) or Jensen-Shannon divergence (JSD) in our model (Eq ??). Unless otherwise stated, $D_f$ is KLD. For further implementation details, refer to Appendix \ref{appen:Implementation}. 

\begin{figure*}[t]
    \centering
    \hfill
    \begin{minipage}{.48\linewidth}
        \centering
        % \vspace{-10pt}
        \setlength\tabcolsep{2.0pt}
        \renewcommand\thetable{2}
        \captionof{table}{
        %Results on an unconditional generation of CIFAR-10.
        %\textbf{Results on a CIFAR-10.}
        \textbf{Image Generation on CIFAR-10.} $\dagger$ indicates the results conducted by ourselves.
        }
        \label{tab:compare-cifar10}
        % \vspace{-2pt}
        \scalebox{0.75}{
            \begin{tabular}{cccc}
            \toprule
            Class & Model &                FID ($\downarrow$)     \\ 
            \midrule
            \multirow{6}{*}{\textbf{GAN}} & SNGAN+DGflow \cite{ansari2020refining} &           9.62    \\
              % & AutoGAN \cite{gong2019autogan} &              12.4    \\
              % & TransGAN \cite{jiang2021transgan} &            9.26       \\
              & StyleGAN2 w/o ADA \cite{karras2020training} &   8.32    \\
              & StyleGAN2 w/ ADA \cite{karras2020training} &       2.92     \\
              &  DDGAN (T=1)\cite{xiao2021tackling}&     16.68  \\
              &  DDGAN \cite{xiao2021tackling}&     3.75   \\
              &    RGM \cite{rgm}             &     \textbf{2.47}   \\
            \midrule
            \multirow{8}{*}{\textbf{Diffusion}}
              &  NCSN \cite{song2019generative}&      25.3       \\
              &  DDPM \cite{ddpm}&                 3.21   \\
              &  Score SDE (VE) \cite{scoresde} &     2.20   \\
              &  Score SDE (VP) \cite{scoresde}&       2.41     \\
              &  DDIM (50 steps) \cite{ddim}&           4.67   \\
              &  CLD \cite{dockhorn2021score} &                   2.25    \\
              &  Subspace Diffusion \cite{jing2022subspace} &      2.17  \\
              &  LSGM \cite{vahdat2021score}&                \textbf{2.10}     \\
            \midrule
            % \multirow{5}{*}{\textbf{VAE\&EBM}} 
            %   & NVAE \cite{vahdat2020nvae} &              23.5       \\
            %   & Glow \cite{kingma2018glow} &                48.9          \\
            %   & PixelCNN \cite{van2016pixel} &             65.9         \\
            %   & VAEBM \cite{xiao2020vaebm} &             12.2       \\
            %   & Recovery EBM \cite{recovery} &     \textbf{9.58}  \\ 
            % \midrule
            \multirow{2}{*}{\textbf{Flow Matching}}
                  &    FM \cite{lipman2022flow}               &   6.35    \\
                  &    OT-CFM \cite{tong2024improving}     &   \textbf{3.74}     \\
            \midrule      
            \multirow{8}{*}{\textbf{OT-based}}
              &    WGAN \cite{wgan}                       &   55.20     \\
              &    WGAN-GP\cite{wgan-gp}            &   39.40     \\
              % & Robust-OT \cite{robust-ot} & 21.57 & - \\
              % &    AE-OT-GAN \cite{ae-ot-gan}       &   17.10     \\
              &    OTM* (\textit{Small}) \cite{otm}      &   21.78   \\
              &    OTM (\textit{Large})$^\dagger$                &   7.68    \\
              &    UOTM (\textit{Small}) \cite{uotm}      &   12.86     \\
              &    UOTM (\textit{Large}) \cite{uotm}      &   2.97$\pm$0.07    \\
              &    Source-fixed UOTM (\textit{Small})$^\dagger$     &   14.4  \\
              &    Source-fixed UOTM (\textit{Large})     &  7.53  \\
              \midrule \multirow{7}{*}{\textbf{WGF-based}}
              &     JKO-Flow \cite{vwgf}        &    23.1   \\
              &     JKO-iFlow \cite{nfjko}        &    29.1   \\
              &    NSGF \cite{nsgf} (\textit{Large})      &   5.55    \\
              &    \textbf{S-JKO} (\textit{Small})$^\dagger$      &   8.78    \\
              &    \textbf{S-JKO} (JSD) (\textit{Small})$^\dagger$      &   \textbf{8.24}   \\
              &    \textbf{S-JKO} (\textit{Large})$^\dagger$      &   \textbf{2.62} $\pm$0.04     \\
              &    \textbf{S-JKO} (JSD) (\textit{Large})$^\dagger$      &   2.66$\pm$0.05     \\
            \bottomrule
            \end{tabular}
            }
    \end{minipage}
    \hfill
    \begin{minipage}{.48\linewidth}
        \centering
%         \setlength\tabcolsep{2.0pt}
%         \captionof{table}{\textbf{Comparison of JKO schemes, UOTMs, and our model.} 
%         Time is a wall-clock training time. 
%         Complexity is the complexity of training time with respect to $K$ and $d$. 
% NFE is the number of function evaluations required to generate a sample.}
%         % \vspace{10pt}
%         \scalebox{0.8}{
%         \begin{tabular}{ccccc}
%             \toprule
%             Model   &  Time & Complexity &  NFE & FID \\
%             \midrule
%             \citet{wgfkorotin}  &  - & $O(K^2d^3)$ & - & -  \\
%             \citet{vwgf} &     $\geq$ 50h & $O(K^2)$ & 160 & 23.1 \\
%             \citet{nfjko}  & $\geq$ 30h  & $O(K^2)$ & $\geq$ 50 & 29.1  \\
%             \midrule
%             fixed-source UOTM (\textit{small}) & 6h & $O(1)$ & 1 & 14.4 \\
%             fixed-source UOTM  (\textit{large}) & 50h & $O(1)$ & 1 & 7.53 \\
%             \midrule
%             Ours (\textit{small})  & 6h & $O(K)$ & 1 & \textbf{8.78} \\
%             Ours (\textit{large})  & 50h & $O(K)$ & 1 & \textbf{2.65} \\
%             \bottomrule
%         \end{tabular}
%         \vspace{7pt}
        \setlength\tabcolsep{2.0pt}
        \renewcommand\thetable{3}
        \centering
        \captionof{table}{
        \textbf{Image Generation on CelebA-HQ.}
        } \label{tab:compare-celeba}
        \vspace{5pt}
        \scalebox{0.75}{
            \begin{tabular}{ccc}
                \toprule
                Class & Model &   FID ($\downarrow$)         \\
                \midrule
                \multirow{6}{*}{\textbf{Diffusion}}
                & Score SDE (VP) \cite{scoresde} &    7.23  \\
                & Probability Flow \cite{scoresde} & 128.13 \\
                & LSGM \cite{vahdat2021score}&       7.22 \\
                & UDM \cite{kim2021score}&         7.16   \\
                & DDGAN \cite{xiao2021tackling} & 7.64    \\
                & RGM \cite{rgm} &  \textbf{7.15}  \\
                \midrule
                \multirow{5}{*}{\textbf{GAN}} 
                & PGGAN \cite{karras2017progressive} &   8.03    \\
                & Adv. LAE \cite{pidhorskyi2020adversarial} &  19.2       \\
                & VQ-GAN \cite{esser2021taming} &    10.2 \\
                & DC-AE \cite{parmar2021dual} &  15.8   \\
                & StyleSwin \citep{zhang2022styleswin} & \textbf{3.25} \\
                \midrule
                \multirow{3}{*}{\textbf{VAE}} 
                & NVAE \cite{vahdat2020nvae} &    29.7    \\
                & NCP-VAE \cite{aneja2021contrastive} &     24.8   \\
                & VAEBM \cite{xiao2020vaebm}&  \textbf{20.4}        \\
                \midrule
                \multirow{2}{*}{\textbf{OT-based}} & UOTM& \textbf{6.36} \\
                & Source-fixed UOTM$^\dagger$ &  7.36\\
                \midrule
                \multirow{2}{*}{\textbf{WGF-based}} & \textbf{S-JKO}$^\dagger$ & 6.40 \\
                 & \textbf{S-JKO (JSD)}$^\dagger$ & \textbf{5.46}\\
                \bottomrule
            \end{tabular}
            % \begin{tabular}{cccc}
            %     \toprule
            %     Class & Model &   FID ($\downarrow$)  & NFE ($\downarrow$) \\
            %     \midrule
            %     \multirow{1}{*}{\textbf{WGF-based}} & 
            %      \textbf{S-JKO (JSD)}$^\dagger$ & \textbf{5.46} & 1 \\
            %      \midrule
            %     \multirow{1}{*}{\textbf{OT-based}} & UOTM \cite{uotm} & \textbf{6.36} & 1 \\
            %     \midrule
            %     \multirow{6}{*}{\textbf{Diffusion}}
            %     & Score SDE (VP) \cite{scoresde} &    7.23 & 4000  \\
            %     & Probability Flow \cite{scoresde} & 128.13 & 335 \\
            %     & LSGM \cite{vahdat2021score}&       7.22 & 23 \\
            %     & UDM \cite{kim2021score}&         7.16 & 2000   \\
            %     & DDGAN \cite{xiao2021tackling} & 7.64 & 4    \\
            %     & RGM \cite{rgm} &  \textbf{7.15} & 4  \\
            %     \midrule
            %     \multirow{5}{*}{\textbf{GAN}} 
            %     & PGGAN \cite{karras2017progressive} &   8.03 & 1   \\
            %     & Adv. LAE \cite{pidhorskyi2020adversarial} &  19.2 & 1       \\
            %     & VQ-GAN \cite{esser2021taming} &    10.2 & 1 \\
            %     & DC-AE \cite{parmar2021dual} &  15.8 & 1   \\
            %     & StyleSwin \citep{zhang2022styleswin} & \textbf{3.25} & 1 \\
            %     \midrule
            %     \multirow{3}{*}{\textbf{VAE}} 
            %     & NVAE \cite{vahdat2020nvae} &    29.7 & 1    \\
            %     & NCP-VAE \cite{aneja2021contrastive} &     24.8  & 1  \\
            %     & VAEBM \cite{xiao2020vaebm}&  \textbf{20.4} & 1       \\
            %     \bottomrule
            % \end{tabular}
            }
    % \vspace{-10pt}
    % \begin{figure}[H]
    %     \centering
    %     \includegraphics[width=0.5\textwidth]{figures/abl_h.png}
    %     \vspace{-10pt}
    %     \caption{Ablation on h for big model?}
    %     \label{fig:abl_h}
    % \end{figure}

    % \vspace{10pt}
        % \hfill
    % \begin{minipage}{.48\linewidth}
        \vspace{5pt}
        \centering
        % \vspace{-10pt}
        % \setlength\tabcolsep{2.0pt}
        \renewcommand\thetable{4}
        \captionof{table}{\textbf{Ablation Study on Phase Number $K$.}} \label{tab:abl_K}
        \vspace{5pt}
            \scalebox{0.8}{
                \begin{tabular}{cccccc}
                    \toprule
                    K   &  10 & 25 & 50 & 100 & 200 \\
                    \midrule
                    S-JKO (KLD) & 2.77 & 2.83 & \textbf{2.62} & 2.73 &   2.67  \\ 
                    S-JKO (JSD) & 3.15 & 3.23 & 2.86 &   2.83 & \textbf{2.66} \\
                    \bottomrule
                \end{tabular}
            }
    \vspace{5pt}
    % \end{minipage}
    % \hfill
        % \begin{minipage}{.48\linewidth}
            \centering
            \renewcommand\thetable{5}
            \centering
            \captionof{table}{\textbf{Ablation Study on Step Size $h$.}} \label{tab:abl_h}
            \vspace{5pt}
            \scalebox{0.8}{
                \begin{tabular}{ccccc}
                    \toprule
                    h   & 0.01 & 0.05 & 0.1 & 0.2 \\ 
                    \midrule
                    S-JKO (KLD) & 2.91 & 2.71 & \textbf{2.62} &   6.11 \\
                    S-JKO (JSD) & 3.82 & 3.03 &   2.83 & \textbf{2.75} \\
                    \bottomrule
                \end{tabular}
                }
        % \end{minipage}
        \hfill
    
    % \begin{figure}[H]
    %     \centering
    %     \includegraphics[width=0.8\textwidth]{figures/abl_K.png}
    %     \vspace{-10pt}
    %     \caption{\textbf{Ablation Study on the number of JKO steps $K$.}
    %     % Ablation on the number of JKO steps $K$.
    %     }
    %     \label{fig:abl_K}
    % \end{figure}
    \end{minipage}
    \hfill
    \vspace{-10pt}
\end{figure*}
\subsection{Distribution Matching on Synthetic Datasets} \label{sec:synthetic}
As described in Sec \ref{sec:relationship}, Source-fixed-UOTM is equivalent to our S-JKO with only one phase training ($K=1$). UOTM variants exhibit prominent scalability \cite{uotm}. However, the limitation of UOTM variants is that they induce inherent distribution errors (Sec \ref{sec:background_uot}). Therefore, \textbf{we evaluate whether our S-JKO can mitigate this distribution error through multi-phase training}. We conducted experiments on two synthetic datasets: Two Circles and 25-Gaussian Mixture. To observe the clear difference, we chose multi-modal datasets as the target datasets, with densities that are spread extensively away from the origin.

Fig \ref{fig:toys} illustrates the generated densities for every 5k training iterations. Note that $k$ on the bottom of the figure indicates the number of JKO steps completed at that training iteration. The generated samples from UOTMs (UOTM and Source-fixed UOTM) tend to be confined to a subset of modes near the origin. In contrast, our model successfully captures and covers all modes present in complex multi-modal distributions. We interpret this phenomenon happens due to the inherent distribution errors of the UOTMs. UOTMs aim to minimize the transport cost between the source and generated distributions. Consequently, \textbf{the generated distribution from UOTMs tends to cluster around the origin}, because the source distribution is a Gaussian distribution centered at the origin. This mode collapse problem of UOTMs can be exacerbated when the target distribution is spread out far from the origin. Meanwhile, the JKO scheme gradually transforms the source distribution into the target distribution, by approximating the Wasserstein Gradient Flow towards the target distribution. As we iterate through the JKO steps, the discrepancy between the generated and the target distributions gradually decreases (Fig \ref{fig:toys}). Therefore, \textbf{our S-JKO mitigates the distribution error of its one-step variant (Source-fixed UOTM) and UOTM}.

\subsection{Scalability on Image Datasets} \label{sec:imagegeneration}
\paragraph{Training Time}
% In this paragraph, we compare the training time for WGF-based algorithms.
In this paragraph, we compare the training times of existing JKO-based algorithms with our model. As aforementioned in Sec. \ref{sec:limWGF}, recent JKO-based approaches \cite{vwgf, nfjko} parametrize the transport map for each JKO step with a separate neural network. This parametrization substantially slows down the training process. On the contrary, as discussed in Sec. \ref{sec:method}, our model effectively reduces training time complexity through the reparametrization trick. To verify whether our model reduces the training time in practice, we measured the wall-clock training time on the single GPU of RTX 3090Ti.

% As aforementioned in Sec. \ref{sec:limWGF}, recent JKO-based approaches \cite{vwgf, nfjko, nsgf} parametrize transport map for each JKO step with a separate neural network, which substantially slows down the training process.
% On the contrary, as discussed in Sec. \ref{sec:method}, due to the reparametrization trick, our model effectively enhances training time complexity. % and significantly reduces the NFEs for inference.
% To verify whether our model genuinely reduces the training time, we measured wall-clock training time on the single GPU of RTX 3090Ti.

As illustrated in Table \ref{tab:JKO-comparison}, training our model with a \textit{Small} backbone only requires 6 hours, which is more than 5 times faster than other comparable JKO-based models. Surprisingly, \textbf{our model with \textit{Small} backbone not only reduces training time but also significantly outperforms other JKO-based methods}, achieving an FID score of 8.78. In contrast, other JKO-based models show FID scores over 20.

% As illustrated in Table \ref{tab:JKO-comparison}, the training \textit{small} model only requires 6 hours of training, which is more than 5 times faster than other existing JKO-based models.
% Surprisingly, our model on \textit{small} architecture not only significantly reduces training time, but also highly surpasses JKO-based methods, achieving the FID score of 8.78.
% % Moreover, our model only requires 1 NFE for inference while comparisons require more than 50 NFEs.

Furthermore, when compared with the Source-fixed UOTM ($K=1$), \textbf{our model exhibits a comparable wall-clock training time to UOTMs}. Specifically, we maintain training time by decreasing the number of iterations per each JKO step ($N$) (See the Appendix \ref{appen:Implementation} for the detail). As a result, our model only requires a similar number of iterations to UOTMs, which is approximately 10K iterations. This efficiency comes from our reparametrization trick that enables convergence within this decreased training iterations.

% Meanwhile, since our model utilizes more JKO steps ($K=50$) in contrast to the Sorce-fixed UOTM ($K=1$), we maintain training time by decreasing the number of iterations per JKO steps ($N$).
% As a result, our model only requires a similar number of iterations with UOTMs, which is approximately 10K iterations.
% Thus, as shown in Table \ref{tab:JKO-comparison} our model exhibits a comparable wall-clock training time to UOTMs.

% In addition, we configured the total number of iterations to range from 80K to 10K iterations by adopting the number of iterations per JKO step ($N$).
% Since the Sorce-fixed UOTM also requires about 10K iterations for training, our model exhibits a comparable wall-clock training time to Sorce-fixed UOTM.
% As demonstrated in Table \ref{tab:JKO-comparison}, these models take approximately 6 and 50 hours to train on \textit{small} and \textit{large} architectures, respectively.
% \vspace{-5pt}
\paragraph{Image Generation}
We assessed our model on two benchmark datasets: CIFAR-10 ($32\times 32$) \cite{cifar10} and CelebA-HQ ($256\times 256$) \cite{celeba}.
% The qualitative performance of our model is presented in Fig ??.
For the quantitative evaluation, we employed the FID \cite{fid} score.
Table \ref{tab:compare-cifar10} shows that \textbf{our model with \textit{Large} backbone demonstrates state-of-the-art results on CIFAR-10 among existing WGF-based models, with an FID of 2.62}. \jw{(See Table \ref{tab:compare-cifar10_full} for a more extensive comparison with various generative models.)}
Our model outperforms the second-best-performing WGF-based model, NSGF \cite{nsgf}, which shows an FID of 5.55, by a significant margin. Note that NSGF employs the same \textit{Large} backbone of NCSN++ \cite{scoresde}.
Furthermore, our model achieves a competitive FID score of 5.46 on CelebA-HQ ($256\times 256$).
To the best of our knowledge, our model is the first WGF-based generative model that has achieved comparable results with state-of-the-art models on image generation tasks, especially on high-resolution image datasets like CelebA-HQ.

% We assessed our model on the two benchmark datasets: CIFAR-10 ($32\times 32$) \cite{cifar10} and CelebA-HQ ($256\times 256$) \cite{celeba}.
% % The qualitative performance of our model is presented in Fig ??.
% For the quantitative evaluation, we employed the FID \cite{fid} score.
% As indicated in Table \ref{tab:compare-cifar10}, our model demonstrates the state-of-the-art results among the existing WGF-based models, with an FID of 2.65.
% Our model outperforms the second-best-performing WGF-based model, NSGF \cite{nsgf}, FID of 5.55, by a significant margin.
% Furthermore, our model achieves a comparable FID score of 6.75 on CelebA-HQ ($256\times 256$).
% To the best of our knowledge, our model is the first WGF-based generative model that has shown comparable results with state-of-the-art models on image generation tasks.

Moreover, to validate the necessity of multiple JKO steps, we compare our model with the Source-fixed UOTM, which is equivalent to a single-JKO step model ($K=1$). \jw{Table \ref{tab:compare-cifar10}} demonstrates that our model outperforms the Source-fixed UOTM in both architectures on CIFAR-10. Furthermore, our model surpasses the Source-fixed UOTM on CelebA-HQ ($256\times256$) by a large margin. \jw{Combining this with the result from Sec. \ref{sec:synthetic}}, we conclude that leveraging multiple JKO steps helps make the generated distribution closer to the target distribution.

% On the other hand, to validate the necessity of multiple JKO steps, we compare our model with the Sorce-fixed UOTM, which is equivalent to a single-JKO step model ($K=1$).
% As shown in Table \ref{tab:JKO-comparison}, the performance of our model highly exceeds the Sorce-fixed UOTM in both architectures on the CIFAR-10 dataset.
% Moreover, our model surpasses the CelebA-HQ ($256\times256$) dataset by a large margin.
% Combining it with generation results in Sec. \ref{sec:synthetic}, leveraging multiple JKO steps facilitates the transport to the target distribution.

% The result of \textif{small} model trained with DCGAN architecture, which is depicted as \textit{small} model, only requires 6 hours of training, which significantly reduce  

% (Big tables/ Qualitative Generation Results, If possible, add precision/recall results to explain that our model successfully matches distributions.) CIFAR10, CelebA-HQ-256 (if possible.) If it is not possible/ train CelebA-64.

% \paragraph{Fast Convergence} 

% \pargraph{Fast Evalutation}

% First, I am pretty sure that our method will converge faster than the JKO scheme (wall clock time). Furthermore, I think our method should converge faster than UOTM since it automatically improves cost function while training (If the results are bad, just erase this paragraph.)

\subsection{Ablation Studies} \label{sec:ablation}
% In this section, we verify the robustness of our model on the main hyperparameters: The number of JKO step $K$ and the step size $h$.
In this section,\textbf{ we conduct ablation studies to assess the robustness of our model on the main hyperparameters for the JKO scheme. }
These parameters include the number of JKO step $K$, the step size $h$, and the functional $\mathcal{F}(\cdot)$.

\paragraph{Ablation on Phase Number $K$} 
We conducted an ablation study on the number of JKO steps $K$. To maintain the total training iterations, we adjusted the number of iterations $N$ per JKO step accordingly. We tested $K\in \{ 10, 25, 50, 100, 200  \}$ for two functionals (KLD and JSD). For each KLD and JSD experiment, we fixed the total number of iterations to 10K and 8K, respectively (See the appendix for details). Table \ref{tab:abl_K} shows that our model with both KLD and JSD shows similar performance across diverse $K$, which demonstrates that our model is robust to $K$. 
Moreover, we observed a marginal improvement in performance as the number of steps increased, achieving FID scores of 2.60 and 2.66 in the KLD and JSD experiments, respectively. 
% This improvement with sufficiently large $K$ aligns with the principles of the Wasserstein Gradient Flow (WGF). This is because a sufficient number of steps are required to converge to a complex target distribution.
We interpret this phenomenon through WGF. A sufficient number of steps are required to converge to the complex data distribution.

% Fig \ref{fig:abl_K} presents ablation studies on the number of JKO steps $K$.
% To maintain the total training iterations, we adjust the number of iterations $N$ per JKO step.
% For each KLD and JSD experiment, we fixed the total number of iterations to 10K and 8K, respectively.
% As Fig \ref{fig:abl_K} depicts, our model on both KLD and JSD shows similar performance under diverse $K$, which indicates that our model is robust to $K$.
% Moreover, there is a slight improvement in performance as the number of steps increases, achieving 2.60 and 2.66 FID scores on KLD and JSD experiments, respectively.
% Aligned with the concept of WGF, employing a sufficiently large $K$ enhances the distributional matching between generated and target distribution.

% We fix other hyperparameters 
\paragraph{Ablation on Step Size $h$}
We performed an ablation study on step size $h$ (Table \ref{tab:abl_h}). We experimented $h \in \{ 0.01, 0.05, 0.1, 0.2 \}$ while fixing $K=50$. Both S-JKOs employing KLD and JSD showed the best results around $h=0.1$ and comparable performance at $h=0.05$. However, the performance on a too-small $h=0.01$ slightly declines. We hypothesize that this is because too small $h$ is insufficient to transport the source distribution to the target distribution, within the fixed number of JKO steps. Moreover, S-JKO-KLD exhibited sharp degradation at $h=0.2$. We interpret this is because of the discretization error of the JKO step. 
Interestingly, this error is much smaller for JSD. Investigating this difference is beyond the scope of this work. However, we believe this would be an interesting future research.
% Fig ?? demonstrates ablation studies on step size $h$.
% Under reasonably small values, the model shows comparable performance.
% However, the performance on a very small $h$ slightly declines.
% We hypothesize that the performance decreases on small $h$ because the small $h$ enforces push-forward distribution (${T_k}_\# \mu$ to remain near the source distribution $\mu$.

\paragraph{Ablation on $f$-divergence}
% Maybe a table? including chisquare
During our ablation study on other hyperparameters, we also examined the impact of $f$-divergence: KLD and JSD. In summary, both $f$-divergences significantly outperform other JKO models and the Source-fixed UOTM on CIFAR-10 (Table \ref{tab:compare-cifar10}), and outperform the Source-fixed UOTM on CelebA-HQ (Table \ref{tab:compare-celeba}). S-JKO-KLD achieves slightly better results than S-JKO-JSD on both datasets. However, S-JKO-JSD is more robust to larger $h$.

\begin{figure}[t]
    \centering
    \begin{minipage}{1\linewidth}
        \vspace{5pt}
        \centering
        % \vspace{-10pt}
        % \setlength\tabcolsep{2.0pt}
        % \renewcommand\thetable{4}
        \captionof{table}{
        \textbf{logSymKL($\downarrow$) between Ground-Truth WGF and Each method at $t=0.5$ for dimensions $d=2, \cdots, 10$.
        }
        }
        \label{tab:numeric_t_0.5}
        \vspace{5pt}
            \scalebox{0.75}{
                \begin{tabular}{cccccc}
                    \toprule
                    Model &	Dual JKO &	EM 50K &EM PR 10K &	ICNN JKO &	Ours \\
                    \midrule
                    $d=2$ &	$-1.4$ &	$-2.1$ &	$-2.0$ &	$\boldsymbol{-2.6}$ &	$-2.3$ \\
                    $d=4$ &	$-0.3$ &	$-1.0$ &	$-0.8$ &	$\boldsymbol{-2.1}$ &	$-0.9$ \\
                    $d=6$ &	$ 0.1 $ &	$ -0.4 $ &	$ -0.2 $ &	$ \boldsymbol{-1.8} $ &	$\boldsymbol{-1.8}$ \\
                    $d=10$ &	$0.6$ &	$0.4$ &	$0.6$ &	$\boldsymbol{-1.8}$ &	$-0.1$ \\
                    \bottomrule
                \end{tabular}
            }
        \vspace{5pt}
        \centering
        \captionof{table}{\textbf{logSymKL($\downarrow$) between Ground-Truth WGF and Each method at $t=0.9$ for dimensions $d=2, \cdots, 10$.}
        } 
        \label{tab:numeric_t_0.9}
        \vspace{5pt}
        \scalebox{0.75}{
            \begin{tabular}{cccccc}
                \toprule
                Model &	Dual JKO &	EM 50K & EM PR 10K &	ICNN JKO &	Ours \\
                \midrule
                $d=2$ &	$-1.1$ &	$-2.3$ &	$-1.9$ &	$\boldsymbol{-2.4}$ &	$\boldsymbol{-2.4}$ \\
                $d=4$ &	$ -0.7 $ &	 $ -1.0 $ &	$ -0.8 $ &	$ \boldsymbol{-2.1} $ &	$ -1.2 $ \\
                $d=6$ &	$ -0.5 $ &	$ -0.3 $ &	$ -0.1 $ &	$ \boldsymbol{-2.2} $ &	$-0.8$ \\
                $d=10$ &	$ 0.1$ &	$ 0.4 $ &	$ 0.4 $ &	$ \boldsymbol{-1.8} $ &	$-0.1$ \\
                \bottomrule
                \end{tabular}
            }
        % \end{minipage}
    \end{minipage}
    \vspace{-10pt}
\end{figure}

\subsection{Numerical Comparison to Ground-Truth WGF} \label{sec:numerical}

\jw{
We numerically compared our model to the ground truth solution of WGF to assess its accuracy. We followed the experimental settings of the Ornstein-Uhlenbeck process experiments in \citet{wgfkorotin}. Specifically, we measured the symmetric KL divergence between the ground-truth solution and the approximate WGF recovered by each method. (See \citet{wgfkorotin} for the details of each method). Table \ref{tab:numeric_t_0.5} and \ref{tab:numeric_t_0.9} present the results. ICNN JKO \citet{wgfkorotin} presents the best symmetric KL divergence to the ground truth solution, while our approach demonstrates the second-best results. However, ICNN JKO requires additional cubic complexity of $O(K^2 d^3)$ for approximating functional $\mathcal{F}(\rho)$ (Table \ref{tab:JKO-comparison}), where $K$ denotes the number of JKO steps and $d$ refers to the data dimension. Hence, it is challenging to apply ICNN JKO to high-dimensional data, such as image datasets in our paper. In this respect, our method provides competitive scalability to high-dimensional datasets, demonstrating a favorable trade-off between scalability and accuracy.
}

\section{Conclusion}
In this paper, we introduce S-JKO, a generative model based on the semi-dual form of the JKO scheme. 
Our work addresses the scalability challenges in previous JKO-based approaches by leveraging (i) the semi-dual form of the JKO scheme, and (ii) by reparametrizing the transport map. 
Additionally, we explore the relationship with UOTM and enhance the distributional matching between the generated distribution and the target distribution. 
Through comprehensive experiments on 2D synthetic datasets and large-scale benchmark datasets like CIFAR-10 and CelebA-HQ, we demonstrate that our proposed model generates high-quality samples in large-scale data while faithfully capturing the underlying distribution.

\section*{Acknowledgements}
This work was supported by KIAS Individual Grant [AP087501] via the Center for AI and Natural Sciences at Korea Institute for Advanced Study, and NRF grant[RS-2024-00421203].

\section*{Impact Statement}
The advancement of image generation techniques carries the significant potential to influence various scientific and industrial domains, such as machine learning, finance, image synthesis, health care, and anomaly detection. 
Our method will improve the models in this area since our method addresses the main challenge of generative models:  
Generating high-quality samples in large-scale datasets while faithfully transporting the distribution to the data. 
As a result, we anticipate that our model can play a role in addressing the negative social impacts associated with existing generative models that struggle to capture the full diversity of data.
On the other hand, the potential negative societal impact of our work is that generative models tend to learn dependencies on the semantics of data, potentially amplifying existing biases. 
Thus, deploying such models in real-world applications necessitates vigilant monitoring to prevent the reinforcement of societal biases present in the data. 
It is crucial to meticulously control the training data and modeling process of generative models to mitigate potential negative societal impacts.

\bibliography{mybib}
\bibliographystyle{icml2024}

\clearpage
\appendix
\onecolumn
\section{Derivation \& Algorithm}
In this section, we provide a derivation of our optimization problem (Eq \ref{eq:reparametrize}) and provide the precise algorithms for our model.

\paragraph{Notations and Asssumptions}
Assume that all the distributions including $\mu$, $\mu_k$ and $\nu$ be absolutely continuous with respect to the Lebesgue measure.
Throughout the paper, we assume that every function $f$ under the subscript of divergence $D$, i.e. $D_f$, is a convex, differentiable, and nonnegative function defined on $\mathbb{R}^+$.
$f^*$ is a convex conjugate of $f$, i.e. $f^*(y) := \sup_{x} \left( \langle x,y\rangle - f(x) \right)$.
For convenience, we define $f^\circ (x) := - f^*(-x)$.
Moreover, we set $c_h(x,y) := \frac{1}{2h} \lVert x-y \rVert_2^2$.

\paragraph{Derivation of the Optimization Problem}
Before the derivation, we start with the following Lemma:
\begin{lemma} \cite{uot1, semi-dual1, semi-dual3, uotm} \label{lemma}
Consider the following optimization problem:
\begin{equation} \label{eq:uot2}
    \inf_{\pi \in \mathcal{M}_+} \int_{\mathcal{X}\times \mathcal{Y}} c_h(x,y) d\pi(x,y)  + D_{\varphi_1}(\pi_0|\mu_k) + D_{\varphi_2}(\pi_1|\nu).
\end{equation}
Then, the semi-dual formulation of Eq \ref{eq:uot2} is given as
\begin{equation} \label{eq:semi-dual-uot}
    \sup_{v \in \mathcal{C}} \int \varphi_1^\circ \left( v^c(x) \right) d\mu_k(x) + \int \varphi_2^\circ (v(y)) d\nu(y),
\end{equation}
where $v^c(x) = \inf_y \left[ c\left(x, y\right) - v(y) \right]$.
Moreover, the strong duality holds.
\end{lemma}
\begin{proof}
    See \citet{uotm} for the proof.
\end{proof}
% \paragraph{Derivation}

This Lemma enables us to reformulate the JKO optimization problem into the semi-dual formulation of UOT.
Suppose $\mu = \mu_0$ and $\nu$ are source and target distributions, respectively, and let $\mathcal{F}(\rho) := D_f \left( \rho | \nu \right)$.
Then, our optimization problem Eq \ref{eq:jko} is as the follows: 
\begin{equation}
    \mu_{k+1} = \underset{\rho \in \mathcal{P}(\mathbb{R}^d)}{\text{argmin}} \underbrace{\frac{1}{2h} \mathcal{W}_2^2 (\mu_{k}, \rho) + D_f (\rho | \nu)}_{\mathcal{L}_\rho}.
\end{equation}
Recall the definition of 2-Wasserstein distance $\mathcal{W}_2$:
\begin{equation}
    \mathcal{W}_2^2 (\rho, \xi) := \min_{\pi \in \Pi(\rho, \xi)} \int_{\mathbb{R}^d \times \mathbb{R}^d} \lVert x-y \rVert_2^2 d\pi(x,y).
\end{equation}
By the definition of $\mathcal{W}_2$, the objective function $\mathcal{L}_\rho$ of Eq \ref{eq:jko-again} can be rewritten as the follows:
% to find a positive joint Radon measure $\pi \in \mathcal{M}_+$ such that minimizes $\int \frac{1}{2h}  \lVert x-y \rVert_2^2 d\pi(x,y) + D_f (\pi_1 | \nu).$ with the restriction $\pi_0 = \mu_k$.
% Formally, Eq \ref{eq:jko-again} can be reformulated as the follows:
\begin{equation} \label{eq:induce-semi-dual-jko2}
    \mathcal{L}_\rho = \min_{\pi \in \Pi(\mu_k, \rho)} \int_{\mathbb{R}^d\times \mathbb{R}^d} c_h (x,y) d\pi +D_f (\pi_1 | \nu).
\end{equation}
Then, by combining Eq \ref{eq:jko-again} and Eq \ref{eq:induce-semi-dual-jko2}, we obtain the following optimization problem:
\begin{equation} \label{eq:jko-uotm}
    \underset{\pi_0 = \mu_k}{\inf} \left[ \int_{\mathbb{R}^d\times \mathbb{R}^d} c_h (x,y) d\pi(x,y) +D_f (\pi_1 | \nu)   \right].
\end{equation}
By configuring $\varphi_1 = \iota$ and $\varphi_2 = f$ in Eq \ref{eq:uot2}, it boils down to Eq \ref{eq:jko-uotm}.
Thus, by applying Lemma \ref{lemma}, the semi-dual of Eq \ref{eq:jko-uotm} is written as follows:
\begin{equation} \label{eq:semi-dual-jko2}
    \sup_{v \in \mathcal{C}} \int v^c(x) d\mu_k(x) + \int f^\circ(v(y)) d\nu(y).
\end{equation}
Note that $v^c(x) := \inf_y ( c(x,y) - v(y) )$.
Equivalently, we can define $v^c(x) := \inf_{\Delta T} ( c(x,\Delta T(x)) - v(\Delta T(x)) ) $.
Thus, Eq \ref{eq:semi-dual-jko2} can be rewritten as follows:
\begin{equation} \label{eq:semi-dual-jko3}
    \sup_{v \in \mathcal{C}} \int \inf_{\Delta T} \left[ c\left(x, \Delta T(x)\right) - v(\Delta T(x)) \right] d\mu_k(x)
    + \int f^\circ(v(y)) d\nu(y).
\end{equation}

Now, suppose $T_k:\mathbb{R}^d \rightarrow \mathbb{R}^d$ be a measurable map such that ${T_k}_\# \mu = \mu_k$.
Then, by reparametrizing  $T := \Delta T \circ T_k $, Eq \ref{eq:semi-dual-jko3} is equivalent to
\begin{equation} \label{eq:reparametrize2}
    \sup_{v \in \mathcal{C}} \int \inf_{T} \left[ c\left(T_k (x),  T (x)\right) - v(T(x)) \right] d\mu(x)
    + \int f^\circ(v(y)) d\nu(y).
\end{equation}

\begin{algorithm}[t]
\caption{Training algorithm of S-JKO with KLD}
    \begin{algorithmic}[1]
    % \Require 
    \REQUIRE Transport network $T_\theta$ and the discriminator network $v_\phi$. 
    \STATE $T_{\text{old}} = \text{Id}$
    \FOR{$k = 0, 1, 2 , \dots, K$}
        \FOR{$i = 0, 1, 2 , \dots, N$}
            \STATE Sample a batch $x\sim \mu$, $y\sim \nu$, and $z\sim \mathcal{N}(\mathbf{0}, \mathbf{I})$.
            \STATE $\hat{y} = T_\theta(x, z)$.
            \STATE Update $\phi$ by using the loss $\mathcal{L}_v$.\\
            \vspace{-10pt}
            $$\mathcal{L}_v = v_\phi (\hat{y}) + f^*\left(-v_\phi (y)\right) + \lambda \lVert \nabla v_\phi(y) \rVert_2^2. $$
            \vspace{-15pt}
            \STATE Sample a batch $x\sim \mu$, and $z_1, z_2 \sim \mathcal{N}(\mathbf{0}, \mathbf{I})$.
            \STATE $\hat{y} = T_\theta(x,z_1)$, $\hat{y}_{\text{old}} = T_{\text{old}}(x,z_2)$.
            \STATE Update $\theta$ by using the loss $\mathcal{L}_T$. \\
            \vspace{-10pt}
            $$ \mathcal{L}_T = c \left(\hat{y}_{\text{old}}, \hat{y}\right) - v_\phi(\hat{y}) $$
            \vspace{-15pt}
        \ENDFOR
        \STATE $T_{\text{old}} \leftarrow T_\theta$
    \ENDFOR
    \end{algorithmic}
\label{alg:kld}
\end{algorithm}

\begin{algorithm}[t]
\caption{Training algorithm of S-JKO with JSD}
    \begin{algorithmic}[1]
    % \Require 
    \REQUIRE Transport network $T_\theta$ and the discriminator network $w_\phi$. 
    \STATE $T_{\text{old}} = \text{Id}$
    \FOR{$k = 0, 1, 2 , \dots, K$}
        \FOR{$i = 0, 1, 2 , \dots, N$}
            \STATE Sample a batch $x\sim \mu$, $y\sim \nu$, and $z\sim \mathcal{N}(\mathbf{0}, \mathbf{I})$.
            \STATE $\hat{y} = T_\theta(x, z)$.
            \STATE Update $\phi$ by using the loss $\mathcal{L}_v$.\\
            \vspace{-10pt}
            $$\mathcal{L}_v = S\left(w_\phi(\hat{y})\right) + S\left( - w_\phi(y) \right) + \lambda \lVert \nabla w_\phi(y) \rVert_2^2. $$
            \vspace{-15pt}
            \STATE Sample a batch $x\sim \mu$, and $z_1, z_2 \sim \mathcal{N}(\mathbf{0}, \mathbf{I})$.
            \STATE $\hat{y} = T_\theta(x,z_1)$, $\hat{y}_{\text{old}} = T_{\text{old}}(x,z_2)$.
            \STATE Update $\theta$ by using the loss $\mathcal{L}_T$. \\
            \vspace{-10pt}
            $$ \mathcal{L}_T = c\left(\hat{y}_{\text{old}}, \hat{y}\right) + S\left(- w_\phi(\hat{y})\right) $$
            \vspace{-15pt}
        \ENDFOR
        \STATE $T_{\text{old}} \leftarrow T_\theta$
    \ENDFOR
    \end{algorithmic}
\label{alg:jsd}
\end{algorithm}

\paragraph{Algorithm}
In image generation tasks, we slightly modify Algorithm \ref{alg:sjko} by following implementations in \citet{uotm, uotmsd}.
We demonstrate the precise training algorithm with KLD in Algorithm \ref{alg:kld}.
As shown in lines 4-5 and lines 7-8, we additionally plugged auxiliary variable $z$ into the network $T_\theta$.
This strategy is known to yield additional improvements in performance \cite{xiao2021tackling, rgm, uotm, uotmsd}.
Moreover, we incorporate $R_1$ regularizer \cite{r1_reg}, i.e. $\lambda \lVert \nabla v_\phi(y)\rVert_2^2$ in line 5, which is a popular regularization employed in various studies \cite{which-converge, xiao2021tackling, rgm, uotm, uotmsd}.
Furthermore, note that the cost function in line 9 is $c(x,y) = \frac{1}{2dh} \lVert x - y \rVert_2^2$.

Algorithm \ref{alg:jsd} demonstrates the exact algorithm for our model with JSD.
In image generation tasks on JSD, we also employ additional auxiliary variables and $R_1$ regularizer.
Suppose $D_f$ is JSD, then the convex conjugate of $f$ is
$$f^*(x) = \begin{cases}
-\log (2-e^x), & \mbox{if } x <\log 2, \\
\infty, & \mbox{if } x  \geq \log 2.
\end{cases}$$
Since the $f^*(v_\phi(y))$ is infinite whenever $v_\phi (y) \geq \log 2$, the reparametrization for $v_\phi$ is inevitable.
Thus, we introduce $w_\phi$, which is a reparametrization of $v_\phi$ as defined as follows:
\begin{equation}
    w_\phi (y) = (\sigma^{-1} \circ \exp) \left( v_\phi(y) - \log 2 \right),
\end{equation}
where $\sigma$ is a sigmoid function.
Then, the objective $\mathcal{L}_v$ and $\mathcal{L}_T$ in lines 5 and 9 in Algorithm \ref{alg:sjko} can be written as the follows:
\begin{align}
    \begin{split}
        \mathcal{L}_v &= S\left(w_\phi(\hat{y})\right) + S\left( - w_\phi(y) \right) \\
        \mathcal{L}_T &= c(\hat{y}_{\text{old}}, \hat{y}) - S\left(w_\phi(\hat{y})\right),    
    \end{split}
\end{align}
where $S(x)$ is a softplus function, i.e. $S(x):= \log (1+e^x)$.
However, technically, it is well-known that the gradient of the generator $T$ saturates when trained with $\mathcal{L}_T $ \cite{gan}.
Thus, by following the technical modification introduced in \citet{gan}, we modify the object as follows:
\begin{align}
    \begin{split}
        \mathcal{L}_v &= S\left(w_\phi(\hat{y})\right) + S\left( - w_\phi(y) \right) \\
        \mathcal{L}_T &= c(\hat{y}_{\text{old}}, \hat{y}) + S\left( - w_\phi(\hat{y})\right).    
    \end{split}
\end{align}
Finally, we obtain the training algorithm of S-JKO with JSD as the Algorithm \ref{alg:jsd}.

\section{Implementation Details} \label{appen:Implementation}
Unless otherwise stated, the source distribution $\mu$ is a $d$-dimensional standard Gaussian distribution and the target distribution $\nu$ is a data distribution.

\subsection{Synthetic data}
\paragraph{Two Circles} 
Suppose $P$ is a uniform distribution on the circles of radius 4 and 8. Then, we generate 2D ``Two Circle" data as follows:
$$ x + 0.2 z \quad x\sim P \quad z \sim \mathcal{N}(0, \mathbf{I}).$$

\paragraph{25-Gaussian Mixture}
Let $P$ be a uniform distribution on $\{ (3i, 3j) : i, j \in\{ -2,-1,0,1,2 \} \}$. Then, we generate 2D ``25-Gaussian Mixture" data as follows:
$$ x + 0.005 z \quad x\sim P \quad z \sim \mathcal{N}(0, \mathbf{I}). $$

\paragraph{Implementation Details}
For all synthetic experiments, we used the same architectures for the transport map $T_\theta$ and the potential network $v_\phi$.
We followed the architectures and hyperparameters of \citet{uotmsd} unless otherwise stated.
We used a batch size of 400 and a learning rate of $10^{-4}$ and $10^{-5}$ for the transport and potential networks, respectively.
We trained the networks for 100K iterations for UOTMs.
For our model, we trained for 20 JKO steps ($K=20$), and 5K iterations for each JKO step ($N=5000$).
Thus, our models are also trained for 100K iterations.
Moreover, we set $h=5$ for the 25-Gaussian Mixture, and $h=2$ for the Two Circles data.
We do not use any regularizations.

\subsection{Image Generation}
Otherwise stated, all the implementation details including preprocessing, hyperparameters, and architectures follow the implementation of \citet{uotmsd}.
The DCGAN model, which is written as \textit{Small} throughout the manuscript, follows the architecture employed in \citet{otm}.
For \textit{Large} model, we follow the implementation of \citet{uotm}.
For all implementations, We employ a batch size of 256, Adam optimizer with $(\beta_1, \beta_2) = (0.5, 0.9)$, and the learning rate of $2\times 10^{-4}$ and $10^{-4}$ for the $T_\theta$ and $v_\phi$ networks, respectively.
Moreover, we used $R_1$ regularization of $\lambda = 0.2$ for CIFAR-10 experiments, and $\lambda = 20$ for CelebA-256 experiments.
For the implementation of our model with KL divergence, we trained for 50 JKO steps ($K=50$), 10K iterations for the first JKO step, and 2K iterations for other JKO steps ($N=2000$). In total, we train for 110K iterations.
For the implementation of our model with Jensen-Shannon divergence, we trained for 35 JKO steps ($K=50$), 10K iterations for the first JKO step, and 2K iterations for other JKO steps ($N=2000$). In total, we train for 80K iterations.
For the implementation of ablation on $K$, we adjusted the number of iterations for each JKO step, i.e. $N$, to fix the total number of training iterations.

\paragraph{Number of network parameters}
In this paragraph, we compare the number of network parameters between comparison models (\cite{vwgf, nfjko}) to our S-JKO on the CIFAR-10 experiments. 
For \textit{Small} and \textit{Large} architecture, we use approximately 0.4M and 48M number of parameters for $T_\theta$, respectively.
\citet{vwgf} employs more than 30M parameters.
Moreover, since \citet{nfjko} use encoder-decoder networks, they can save the number of parameters to approximately 2-3M.

\paragraph{Evaluation Metric}
We used 50,000 generated samples to measure FID \cite{fid} scores.

\section{Additional Results}

\subsection{Training dynamics}
Through Fig \ref{fig:traj_KLD} and \ref{fig:traj_JSD}, we visualize the trajectories of S-JKO trained on CIFAR-10.
We sampled a batch $x\sim \mu$ and visualized $\{T_{5j+1} (x)\}$ for a non-negative integer $j$.

\subsection{Additional Qualitative Results}
Through Fig \ref{fig:appen_kld_cifar10_large}-\ref{fig:celeba_jsd}, we present generated samples for S-JKO trained on CIFAR-10 and CelebA-HQ ($256\times 256$).

\begin{figure}[t]
    \centering
    \includegraphics[width=.5\textwidth]{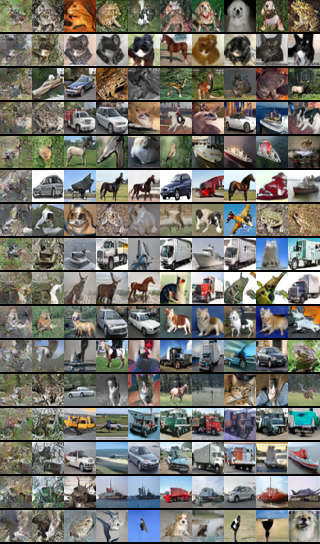}
    \caption{\textbf{CIFAR-10 trajectories from S-JKO (KLD)} for $K=5j+1$ ($0\leq j \leq 9$).}
    \label{fig:traj_KLD}
\end{figure}

\begin{figure}[t] 
    \centering
    \includegraphics[width=.5\textwidth]{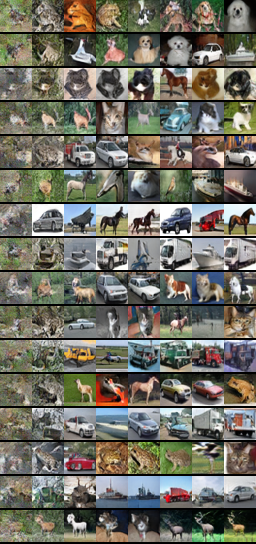}
    \caption{\textbf{CIFAR-10 trajectories from S-JKO (JSD)} for $K=5j+1$ ($0\leq j \leq 7$).}
    \label{fig:traj_JSD}
\end{figure}

\clearpage

\begin{figure}[t] 
    \centering
    \includegraphics[width=.86\textwidth]{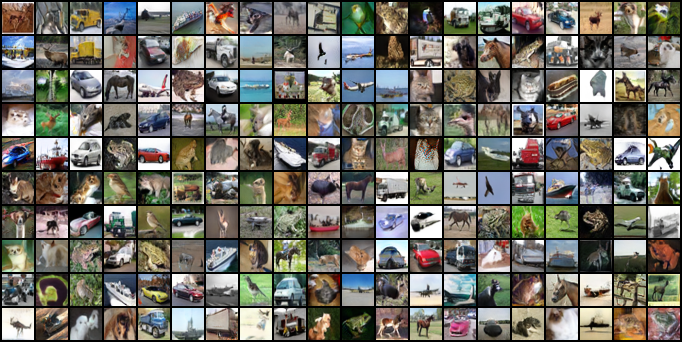}
    \caption{\textbf{Generated samples from S-JKO (KLD)} trained on CIFAR-10 ($32\times 32$) with \textit{Large} model.}
    \label{fig:appen_kld_cifar10_large}
\end{figure}

\begin{figure}[t] 
    \centering
    \includegraphics[width=.86\textwidth]{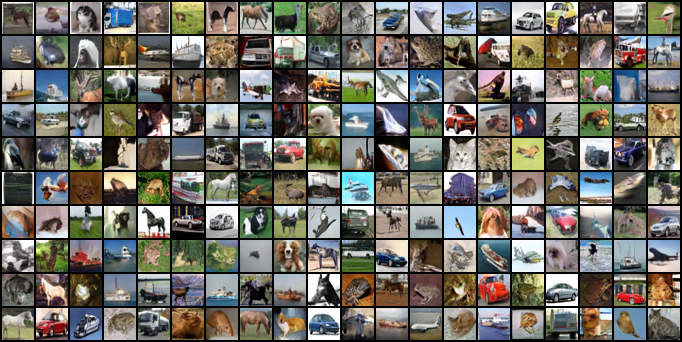}
    \caption{\textbf{Generated samples from S-JKO (JSD)} trained on CIFAR-10 ($32\times 32$) with \textit{Large} model.}
    \label{fig:appen_jsd_cifar10_large}
\end{figure}

\begin{figure}[t]
    \centering
    \includegraphics[width=.86\textwidth]{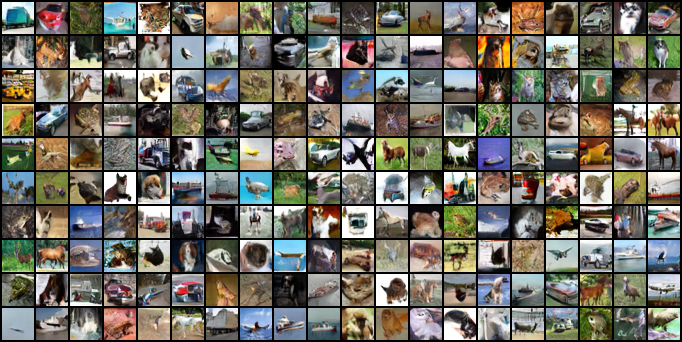}
    \caption{\textbf{Generated samples from S-JKO (KLD)} trained on CIFAR-10 ($32\times 32$) with \textit{Small} model.}
     \label{fig:appen_kld_cifar10_small}
\end{figure}

\begin{figure}[t] 
    \centering
    \includegraphics[width=.86\textwidth]{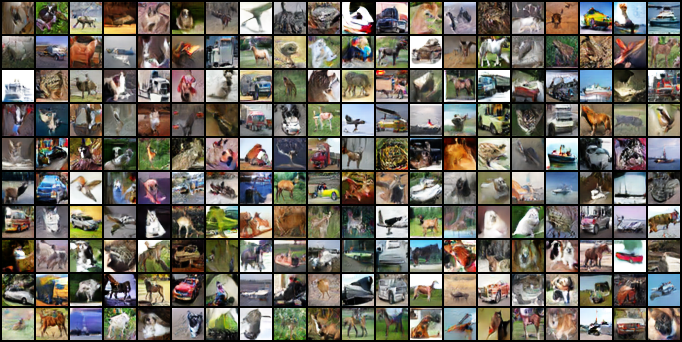}
    \caption{\textbf{Generated samples from S-JKO (JSD)} trained on CIFAR-10 ($32\times 32$) with \textit{Small} model.}
    \label{fig:appen_jsd_cifar10_small}
\end{figure}

\begin{figure}[t]
    \centering
    \includegraphics[width=.95\textwidth]{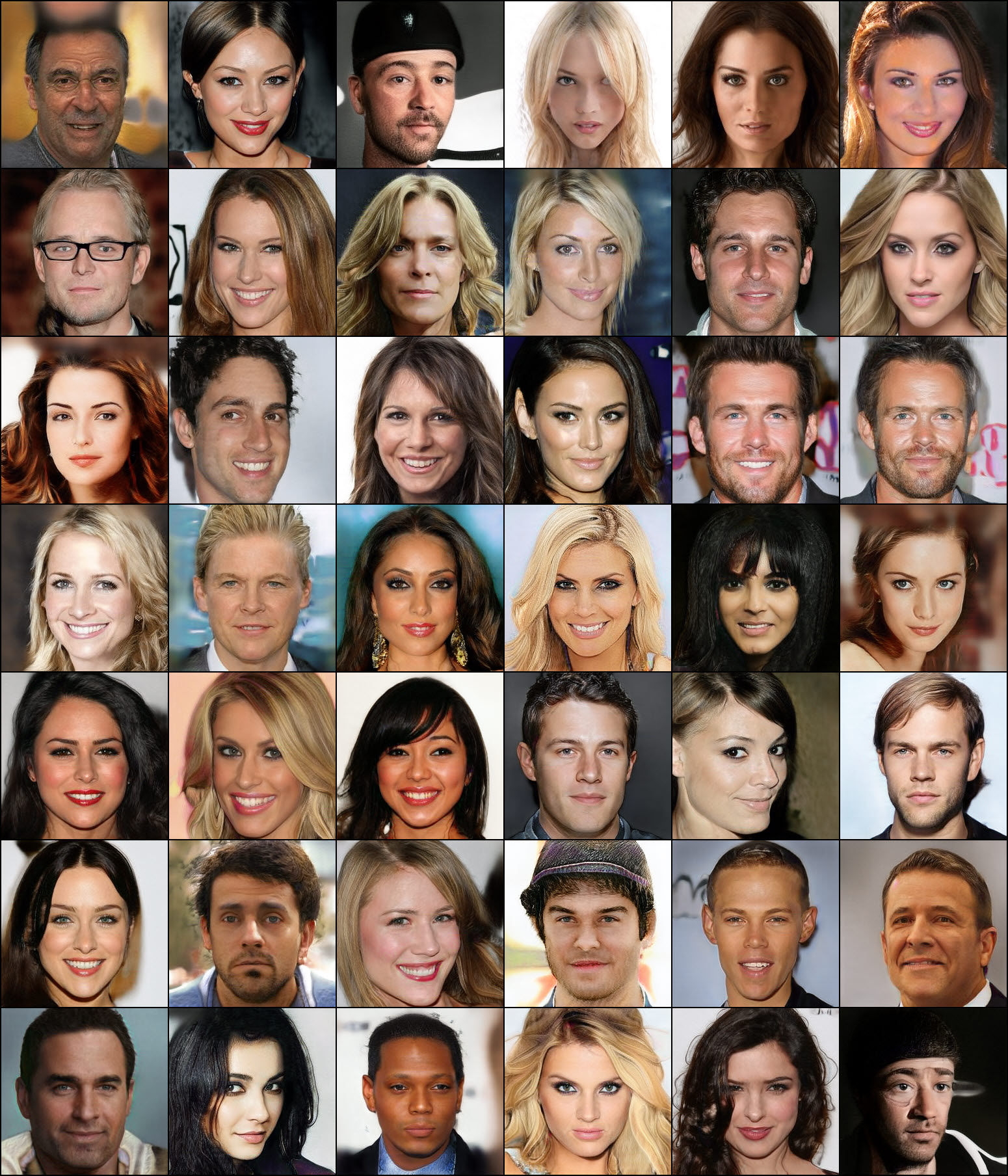}
    \caption{\textbf{Generated samples from S-JKO (KLD)} trained on CelebA-HQ ($256\times 256$).}
     \label{fig:celeba_kl}
\end{figure}

\begin{figure}[t] 
    \centering
    \includegraphics[width=.95\textwidth]{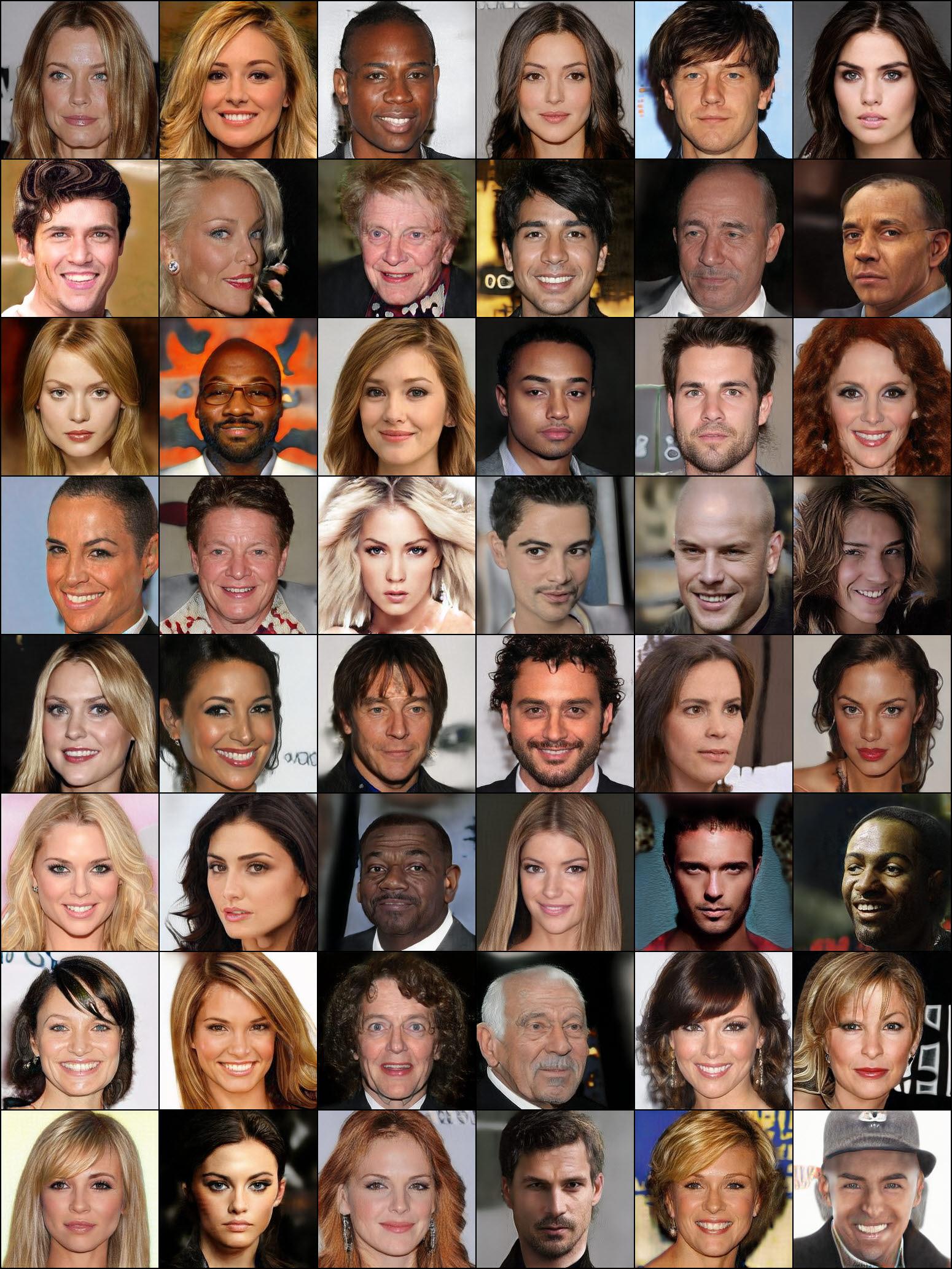}
    \caption{\textbf{Generated samples from S-JKO (JSD)} trained on CelebA-HQ ($256\times 256$).}
    \label{fig:celeba_jsd}
\end{figure}

\begin{figure*}[t]
    \centering
    \begin{minipage}{.8\linewidth}
        \centering
        % \vspace{-10pt}
        \setlength\tabcolsep{2.0pt}
        \renewcommand\thetable{2}
        \captionof{table}{
        %Results on an unconditional generation of CIFAR-10.
        %\textbf{Results on a CIFAR-10.}
        \textbf{Extensive Comparison with Diverse Generative Models on Image Generation on CIFAR-10.} $\dagger$ indicates the results conducted by ourselves.
        }
        \label{tab:compare-cifar10_full}
        % \vspace{-2pt}
        \scalebox{1}{
            \begin{tabular}{cccc}
            \toprule
            Class & Model &                FID ($\downarrow$)     \\ 
            \midrule
            \multirow{8}{*}{\textbf{GAN}} & SNGAN+DGflow \cite{ansari2020refining} &           9.62    \\
              & AutoGAN \cite{gong2019autogan} &              12.4    \\
              & TransGAN \cite{jiang2021transgan} &            9.26       \\
              & StyleGAN2 w/o ADA \cite{karras2020training} &   8.32    \\
              & StyleGAN2 w/ ADA \cite{karras2020training} &       2.92     \\
              &  DDGAN (T=1)\cite{xiao2021tackling}&     16.68  \\
              &  DDGAN \cite{xiao2021tackling}&     3.75   \\
              &    RGM \cite{rgm}             &     \textbf{2.47}   \\
            \midrule
            \multirow{8}{*}{\textbf{Diffusion}}
              &  NCSN \cite{song2019generative}&      25.3       \\
              &  DDPM \cite{ddpm}&                 3.21   \\
              &  Score SDE (VE) \cite{scoresde} &     2.20   \\
              &  Score SDE (VP) \cite{scoresde}&       2.41     \\
              &  DDIM (50 steps) \cite{ddim}&           4.67   \\
              &  CLD \cite{dockhorn2021score} &                   2.25    \\
              &  Subspace Diffusion \cite{jing2022subspace} &      2.17  \\
              &  LSGM \cite{vahdat2021score}&                \textbf{2.10}     \\
            \midrule
            \multirow{2}{*}{\textbf{Flow Matching}}
                  &    FM \cite{lipman2022flow}               &   6.35    \\
                  &    OT-CFM \cite{tong2024improving}     &   \textbf{3.74}     \\
            \midrule    
            \multirow{6}{*}{\textbf{VAE\&EBM}} 
              & NVAE \cite{vahdat2020nvae} &              23.5       \\
              & Glow \cite{kingma2018glow} &                48.9          \\
              & PixelCNN \cite{van2016pixel} &             65.9         \\
              & VAEBM \cite{xiao2020vaebm} &             12.2       \\
              & Recovery EBM \cite{recovery} &     9.58  \\ 
              & CDRL-large \cite{coorp} &     \textbf{3.68}  \\ 
            \midrule
            \multirow{5}{*}{\textbf{OT-based}}
              &    WGAN \cite{wgan}                       &   55.20     \\
              &    WGAN-GP\cite{wgan-gp}            &   39.40     \\
              % & Robust-OT \cite{robust-ot} & 21.57 & - \\
              % &    AE-OT-GAN \cite{ae-ot-gan}       &   17.10     \\
              &    OTM (\textit{Large})$^\dagger$                &   7.68    \\
              &    Source-fixed UOTM (\textit{Large})     &  7.53  \\
              &    UOTM (\textit{Large}) \cite{uotm}      &   \textbf{2.97}    \\
              \midrule \multirow{5}{*}{\textbf{WGF-based}}
              &     JKO-Flow \cite{vwgf}        &    23.1   \\
              &     JKO-iFlow \cite{nfjko}        &    29.1   \\
              &    NSGF \cite{nsgf} (\textit{Large})      &   5.55    \\
              &    \textbf{S-JKO} (\textit{Large})$^\dagger$      &   \textbf{2.62}      \\
              &    \textbf{S-JKO} (JSD) (\textit{Large})$^\dagger$      &   2.66     \\
            \bottomrule
            \end{tabular}
            }
    \end{minipage}
\end{figure*}

\end{document}